\documentclass[onefignum,onetabnum]{siamonline220329}

\usepackage[utf8]{inputenc}
\usepackage[mathscr]{euscript}
\usepackage{graphicx,bm,stmaryrd,mathtools}
\usepackage{amsmath,amsfonts,amssymb,listings,bbm}
\usepackage{caption}
\usepackage{subcaption,color,gensymb}
\usepackage{tcolorbox}
\usepackage{afterpage}
\usepackage{cleveref}

\usepackage{float}

\usepackage[shortlabels]{enumitem}

\newsiamthm{prop}{Proposition}
\newsiamremark{remark}{Remark}
\newsiamthm{assum}{Assumption}

\DeclareMathOperator*{\argmin}{\arg\!\min}

\DeclareMathOperator{\E}{\mathbb{E}}
\DeclareMathOperator{\R}{\mathbb{R}}
\DeclareMathOperator{\Prob}{\mathbb{P}}

\DeclareMathOperator{\toe}{\xrightarrow[n \to \infty]{e}}

\DeclareMathOperator{\exc}{\mathrm{exc}}
\DeclareMathOperator{\epi}{\mathrm{epi}}
\DeclareMathOperator{\dom}{\mathrm{dom}}
\DeclareMathOperator{\inti}{\mathrm{int}}

\usepackage[margin=1in]{geometry}

\title{Data-Driven Priors in the Maximum Entropy on the Mean Method for Linear Inverse Problems.}
\author{Matthew King-Roskamp, Rustum Choksi, \and Tim Hoheisel}
\date{\today}

\begin{document}

\maketitle

\begin{abstract}
We establish the theoretical framework for implementing the maximumn entropy on the mean (MEM) method for linear inverse problems in the setting of approximate (data-driven) priors. We prove a.s. convergence for empirical means and further develop general estimates for the difference between the MEM solutions with different priors $\mu$ and $\nu$ based upon the epigraphical distance between their respective log-moment generating functions. These estimates allow us to establish a rate of convergence in expectation for empirical means. We illustrate our results with denoising on MNIST and Fashion-MNIST data sets.
\end{abstract}

\section{Introduction}

 Linear inverse problems are pervasive in data science. A canonical example (and our motivation here) is denoising and deblurring in image processing. 
 Machine learning algorithms, particularly neural networks trained on large data sets, have proven to be a game changer in solving these problems. 
 However, most machine learning algorithms suffer from the lack of a foundational framework upon which to rigorously assess their performance. Thus there is a need for mathematical models which are on one end, data driven, and on the other end, open to rigorous evaluation. In this article, we address one such model: the {\it Maximum Entropy on the Mean} (MEM).  In addition to providing the theoretical framework, we provide several numerical examples for denoising images from {\it MNIST} \cite{deng2012mnist} and {\it Fashion-MNIST} \cite{xiao2017fashion} data sets.

Emerging from ideas of E.T. Jaynes in 1957
\cite{jaynes1957information1,jaynes1957information2}, 
various forms and interpretations of MEM have appeared in the literature and found applications in different disciplines (see  \cite{le1999new, vaisbourd2022maximum} and the references therein). The MEM method has recently been demonstrated to be a powerful tool for 
 the blind deblurring of images possessing some form of symbology (e.g., QR barcodes)  \cite{8758192,rioux2020maximum}. 

Let us briefly summarize the MEM method for linear inverse problems, with full details provided in the next section. Our canonical inverse problem takes the following form 
\begin{equation}\label{lin-inverse-p}
    b = C\overline{x} + \eta. 
\end{equation}
The unknown solution $\overline{x}$  is a vector in $\R^{d}$; the  observed data is $b \in \R^{m}$;  $C \in \R^{m \times d}$,  and  $\eta \sim \mathcal{Z}$ is an random noise vector in $\R^{m}$ drawn from noise distribution $\mathcal{Z}$. In the setting of image processing,  $\overline{x}$ denotes the the ground truth image with $d$ pixels, $C$ is a blurring matrix with typically $d = m$, and the observed noisy (and blurred image) is $b$. 
For known $C$, we seek to recover the ground truth $\overline{x}$ from $b$. In certain classes of images, the case where $C$ is also unknown (blind deblurring) can also be solved with the MEM framework (cf. \cite{8758192,rioux2020maximum}) but we will not focus on this here. In fact, our numerical experiments will later focus purely on denoising, i.e., $C = I$. 
 The power of MEM is to exploit the fact that there exists a prior distribution $\mu$ for the space of admissible ground truths. 
 The basis of the method is the {\it MEM function} 
 $\kappa_{\mu} :\R^{d} \to \R \cup \{ + \infty\}$  defined as 
\begin{equation*}
    \kappa_{\mu}(x) := \inf \left\{ \mathrm{KL}(Q \Vert \mu) \: : \:  Q \in \mathcal{P}(\mathcal{X}), \E_{Q} = x \right\},
\end{equation*}
 where $\mathrm{KL}(Q \Vert \mu)$  denotes the Kullback-Leibler (KL) divergence between the probability distributions $\mu$ and $\nu$ (see \Cref{sec:MEMProblem} for the definition). 
With $\kappa_{\mu}$ in hand, our proposed solution to 
\cref{lin-inverse-p} is 
\begin{equation}\label{MEM-sol}
    \overline{x}_{\mu} = \argmin_{x \in \R^{d}}  \left\{ \alpha  g_b(Cx) \, +\,   \kappa_{\mu}(x) \right\}, 
\end{equation}
where $g_{b}$ is any (closed, proper) convex function that measures  {\it fidelity} of $Cx$ to $b$. The function $g_{b}$ depends on $b$ and can in principle be adapted to the noise distribution $\mathcal Z$. For example, as was highlighted in \cite{vaisbourd2022maximum}, one can take the {\it MEM estimator} (an alternative to the well-known {\it maximum likelihood estimator}) based upon a family of distributions (for instance, if the noise is Gaussian, then the MEM estimator is the familiar $g_b (\cdot) = \frac{1}{2}\Vert  (\cdot)  - b\Vert_{2}^{2}$). Finally $\alpha >0$ is a fidelity parameter.  

The variational problem \cref{MEM-sol} is solved via its Fenchel dual. As we explain in \Cref{sec:MEMProblem}, we exploit the well-known connection in the large deviations literature that, under appropriate assumptions,  the MEM function  $\kappa_{\mu}$ is simply the Cram\'er rate function defined as the Fenchel conjugate of the log-moment generating function (LMGF)
\begin{equation*}
    L_{\mu}(y): = \log \int_{\mathcal{X}} \exp\langle y, \cdot \rangle d\mu.
\end{equation*}
Under certain assumptions on $g_b$ (cf. \Cref{sec:MEMProblem}) we obtain strong duality
\begin{equation}\label{dual-primal}
   \min_{x \in \R^{d}}  \alpha g_b(Cx) +  \kappa_{\mu}(x) = - \min_{z \in \R^{m}} \alpha g^{*}(-z/\alpha) + L_{\mu}(C^{T}z),
\end{equation}
and, more importantly,  a primal-dual recovery is readily available: If 
$\overline z_{\mu}$ is a solution to the dual problem (the argmin of the right-hand-side of \cref{dual-primal}) 
 then 
\[    \overline{x}_{\mu} := \nabla L_{\mu}(C^{T}\overline{z})
\] 
is the unique solution of the primal problem. This is the MEM method in a nutshell. 

In this article, we address the following question: Given an approximating sequence $\mu_n \to \mu$ (for example, one generated by a sample empirical mean of $\mu$), does the approximate MEM solution $\overline x_{\mu_n}$ converge to the solution $\overline x_\mu$, and if so, at which rate? A key feature of the MEM approach is that one does not have to learn the full distribution $\mu$ from samples, but rather only approximate the LMGF $L_{\mu}$. Hence, our analysis is based on the {\it closeness} of $L_{\mu_n}$ to $L_{\mu}$ resulting in the {\it closeness} of the dual solutions $\overline z_n$ and in turn the primal solutions $\overline x_{\mu_n}$. Here,  we leverage the fundamental work of Wets et al.  on {\it epigraphical distances, epigraphical convergence, and epi-consistency} (\cite{rockafellar2009variational},\cite{royset2022optimization},\cite{king1991epi}). 
 
 Our results are presented in four sections.  In \Cref{sec:epi-convergence}, we work with a general $g_b$ satisfying standard assumptions. 
 We consider the simplest way of approximating $\mu$ via  empirical means of $n$ i.i.d. samples from $\mu$.  
 In \cref{Thm:convergence_of_primal}, we prove that  the associated MEM solutions  $\overline{x}_{\mu_n}$ converge almost surely to the solution $\overline{x}_{\mu}$ with full prior. In fact, we prove a slightly stronger result pertaining to $\varepsilon_n$-solutions as $\varepsilon_n\downarrow 0$. This result opens the door to two natural questions: (i) 
 At which rate do the solutions converge? 
(ii) Empirical means is perhaps the simplest way of approximating $\mu$ and will inevitably yield a rate dictated by the law of large numbers. 
Given that the MEM method rests entirely on the LMGF of the prior, it is natural to ask how the rate depends on an approximation to the LMGF. So, 
if we used a different way of approximating $\mu$, what would the rate look like? 
We address these questions for the case $g_b  = \frac{1}{2}\Vert  (\cdot)  - b\Vert_{2}^{2}$. In \Cref{sec:rates} we provide insight into the second question first via a deterministic estimate which controls the difference in the respective solutions associated with two priors $\nu$ and $\mu$ based upon the epigraphical distance between their respective LMGFs. We again prove a general result for $\varepsilon$-solutions associated with prior $\mu$ (cf.
\cref{thm:epsdeltaprimalbound_full}). In \Cref{sec:rates_n_empirical}, we apply this bound to the particular case of the empirical means approximation, proving a $\frac{1}{n^{1/4}}$ convergence rate (cf. \Cref{thm:final_rate_n}) in expectation. 

Finally, in \Cref{sec:numerics}, we present several numerical experiments for denoising based upon a finite MNIST data set.
These serve not to compete with any of the state of the art machine learning-based denoising algorithm, but rather to highlight the effectiveness of our data-driven mathematical model which is fully supported by theory.  

\begin{remark}[Working at the higher level of the probability distribution of the solution] \label{remark:measure_valued}
{\rm As in \cite{8758192,rioux2020maximum}, an equivalent formulation of the MEM problem is to work not at the level of the $x$, but rather at the level of the probability distribution of the ground truth, i.e., we seek to solve
 \[    \overline{Q} \, = \, { \argmin}_{Q \in \mathcal{P}(\mathcal{X})} \, \,  \left\{  \alpha g_b(C \mathbb{E}_Q) \, + \, \mathrm{KL}(Q \Vert \mu) \right\},   \]
 where one can recover the previous image-level solution as $\overline x_\mu = \mathbb{E}_{\overline Q}$.
 As shown in \cite{rioux2020maximum}, under appropriate assumptions this reformulated problem has exactly the same dual formulation as in the right-hand-side of \cref{dual-primal}. Because of this one has full access to the entire probability distribution of the solution, not just its expectation. This proves  useful in our MNIST experiments where the optimal $\nu$ is simply a weighted sum of images uniformly sampled from the MNIST set. 
 For example, one can do thresholding (or masking) at the level of the optimal $\nu$ (cf. the examples in \Cref{sec:numerics}). 
}
\end{remark}

\noindent
{\it Notation:} $\overline{\R} := \R \cup \{\pm \infty \}$ is the extended real line. The standard inner product on $\R^n$ is   $\langle \cdot, \cdot \rangle$ and $\|\cdot\|$ is the Euclidean norm. For  $C \in \R^{m \times d}$, $\Vert C \Vert = \sqrt{\lambda_{\max}(C^{T}C)}$ is its spectral norm, and analagously $\sigma_{\min}(C) = \sqrt{\lambda_{\min}(C^{T}C)}$ is the smallest singular value of $C$.  The trace of $C$ is denoted $\text{Tr}(C)$. For smooth $f : \R^{d} \to \R$, we denote its gradient and Hessian by $\nabla f$ and $\nabla^{2} f$, respectively.

\section{Tools from convex analysis and the MEM method for solving the problem \cref{lin-inverse-p} }

\subsection{Convex analysis}
\label{sec:convexAnalysisPrelim}
 We present here  the tools from convex analysis essential to our study. We refer the reader to the standard texts by Bauschke and Combettes \cite{bauschke2019convex} or Chapters 2 and 11 of Rockafellar and Wets \cite{rockafellar2009variational} for further details. Let $f:\R^d \to\overline{\R}$.  The domain of $f$ is $\text{dom}(f):=\{ x \in \R^{d} \: \vert \: f(x) < + \infty \}$. We call  $f$ proper if $\dom(f)$ is nonempty and  $f(x) > - \infty$ for all $x$. We say that  $f$ is lower semicontinuous (lsc) if $f^{-1}([-\infty, a])$ is closed (possibly empty) for all $a \in \R$. We  define the (Fenchel) conjugate $f^{*} :\R^{d} \to \overline{\R}$ of $f$ as $f^{*}(x^{*}) := \sup_{x \in \R^{d}} \{ \langle x, x^{*} \rangle - f(x) \}.$ A proper $f$ is said to be convex, if 
 $f(\lambda x + (1-\lambda) y) \leq \lambda f(x) + (1-\lambda) f(y)$ for every $x,y \in \text{dom}(f)$ and all $\lambda \in (0,1)$.
If the former inequality is strict for all $x \neq y$, then $f$ is said to be strictly convex. Finally, if $f$ is proper and there is a $c>0$ such that
$f-\frac{c}{2}\|\cdot\|^2$ is convex
we say $f$ is $c$-strongly convex. 
In the case where $f$ is (continuously) differentiable on $\R^{d}$, then $f$ is $c$-strongly convex if and only if 
\begin{equation}
f(y) - f(x) \geq \nabla f(x)^{T}(y-x) + \frac{c}{2} \Vert y-x \Vert_{2}^{2}\quad\forall x,y\in \R^d. \label{eqn:alternate_strongconvexity}
\end{equation} 
The subdifferential  of $f :\R^{d} \to \overline \R$ at $\overline{x}$ is the 
    $\partial f(\overline{x}) = \{ x^{*} \in \R^{d} \: \vert \: \langle x-\overline{x},x^{*}\rangle \leq f(x) - f(\overline{x}), \: \forall x \in \R^{d} \}.$
A function $f : \R^{d} \to \overline{\R}$ is said to be level-bounded if for every $\alpha \in \R$, the set $f^{-1}([-\infty, \alpha])$ is bounded (possibly empty). $f$ is (level) coercive if it is bounded below on bounded sets and satisfies 
\begin{equation*}
    \liminf_{\Vert x \Vert \to \infty} \frac{f(x)}{\Vert x \Vert} > 0.
\end{equation*}
In the case $f$ is proper, lsc, and convex, level-boundedness is equivalent to level-coerciveness \cite[Corollary 3.27]{rockafellar2009variational}. $f$ is said to be supercoercive if $\liminf_{\Vert x \Vert \to \infty} \frac{f(x)}{\Vert x \Vert} =+\infty$. \\
A point $\overline{x}$ is said to be an $\varepsilon$-minimizer of $f$ if $f(\overline{x}) \leq \inf_{x \in \R^{d}} f(x) + \varepsilon$
for some $\varepsilon >0$. We denote the set of all such points as $S_{\varepsilon}(f)$. Correspondingly, the  solution set of $f$ is denoted as $\argmin(f) = S_{0}(f) =: S(f).$

The epigraph of a function $f : \R^{d} \to \overline{\R}$ is the set $\text{epi}(f) := \{ (x,\alpha) \in \R^{d} \times \overline{\R}  \: \vert \: \alpha \geq f(x) \}$. A sequence of functions $f_{n} : \R^{d} \to \overline{\R} $ epigraphically converges (epi-converges)\footnote{This is one of many equivalent conditions that characterize epi-convergence, see e.g. \cite[Proposition 7.2]{rockafellar2009variational}.} to $f$, written $f_{n} \toe f$, if and only if
\[
(i)\: \forall z,  \forall z_{n} \to z: \: \liminf f_{n}(z_{n}) \geq f(z), \quad 
    (ii)\:   \forall z \;\exists z_{n} \to z:  \limsup f_{n}(z_{n})\leq f(z). 
\]

\subsection{Maximum Entropy on the Mean Problem} \label{sec:MEMProblem}

For basic concepts of measure and probability, we follow most closely the standard text of Billingsley \cite[Chapter 2]{billingsley2017probability}. Globally in this work, $\mu$ will be a Borel probability measure defined on compact $\mathcal{X} \subset \R^{d}$\footnote{Equivalently, we could work with a Borel measure $\mu$ on $\R^d$ with support contained in $\mathcal X$.}. Precisely, we work on the probability space $(\mathcal{X},\mathcal{B}_{\mathcal{X}}, \mu)$, where $\mathcal{X} \subset \R^{d}$ is compact and $\mathcal{B}_{\mathcal{X}} = \{ B \cap \mathcal{X} \: : \: B \in \mathbb B_d \}$ where $\mathbb B_d$ is the $\sigma$-algebra induced by the open sets in $\R^d$. 
We will denote the set of all probability measures on the measurable space $(\mathcal{X},\mathcal{B}_{\mathcal{X}})$ as $\mathcal{P}(\mathcal{X})$, and refer to elements of $\mathcal{P}(\mathcal{X})$ as probability measures on $\mathcal{X}$, with the implicit understanding that these are always Borel measures. For $Q,\mu\in \mathcal P(\mathcal X)$, we say $Q$ is absolutely continuous with respect to $\mu$ (and write $Q \ll \mu$) if for all $A \in \mathcal{B}_{\mathcal{X}}$ with $\mu(A) = 0$, then $Q(A) = 0$. \cite[p.~422]{billingsley2017probability}. For $Q \ll \mu$, the Radon-Nikodym derivative of $Q$ with respect to $\mu$ is defined as the (a.e.) unique function $\frac{dQ}{d\mu}$ with the property $Q(A) = \int_{A} \frac{dQ}{d\mu} d\mu$  for $A\in \mathcal B_{\mathcal X}$ \cite[Theorem 32.3]{billingsley2017probability}.

The Kullback-Leibler (KL) divergence \cite{kullback1951information} of $Q \in \mathcal{P}(\mathcal{X})$ with respect to $\mu \in \mathcal{P}(\mathcal{X})$ is defined as
\begin{equation}
    \text{KL}(Q\Vert \mu) := \begin{cases} \int_{\mathcal{X}} \log(\frac{dQ}{d\mu}) d \mu, & Q \ll \mu, \\
    + \infty, & \text{ otherwise.} \end{cases} \label{def-KL}
\end{equation}
For $\mu \in \mathcal{P}(\mathcal{X})$, the expected value $\E_{\mu} \in \R^{d}$ and moment generating function $M_{\mu}: \R^{d} \to \R$ function of $\mu$ are defined as \cite[Ch.21]{billingsley2017probability}
\begin{equation*}
    \E_{\mu} := \int_{\mathcal{X}}x d\mu(x) \in \R^{d},\qquad  M_{\mu}(y) := \int_{\mathcal{X}} \exp\langle y, \cdot \rangle d\mu,
\end{equation*}
respectively. The log-moment generating function  of $\mu$ is defined as
\begin{equation*}
    L_{\mu}(y):= \log M_{\mu}(y) = \log \int_{\mathcal{X}} \exp\langle y, \cdot \rangle d\mu .
\end{equation*}
As $\mathcal{X}$ is bounded, $M_{\mu}$ is finite-valued everywhere. By standard properties of moment generating functions (see e.g. \cite[Theorem 4.8]{severini2005elements}) it is then analytic everywhere, and in turn so is $L_{\mu}$.  

Given $\mu \in \mathcal{P}(\mathcal{X})$, the Maximum Entropy on the Mean (MEM) function \cite{vaisbourd2022maximum} $\kappa_{\mu} :\R^{d} \to \overline{\R}$ is 
\begin{equation*}
    \kappa_{\mu}(y) := \inf\{ \mathrm{KL}(Q \: \Vert \: \mu) : \E_{Q} = y , Q \in \mathcal{P}(\mathcal{X}) \} .
\end{equation*}
The functions $\kappa_{\mu}$ and $L_{\mu}$ are paired in duality in a way that is fundamental to this work. We will flesh out this connection, as well as give additional properties of $\kappa_{\mu}$ for our setting; a Borel probability measure $\mu$ on compact $\mathcal{X}$. A detailed discussion of this connection under more general assumptions is the subject of \cite{vaisbourd2022maximum}.

For any $\mu \in \mathcal{P}(\mathcal{X})$ we have a vacuous tail-decay condition of the following form: for any $\sigma >0$,
\begin{equation*}
\int_{\mathcal{X}} e^{\sigma \Vert x \Vert} d\mu(x) \leq \max_{x \in \mathcal{X} } \Vert x \Vert e^{\sigma \Vert x \Vert} < + \infty.
\end{equation*}
Consequently,  by \cite[Theorem 5.2 (iv)]{donsker1976asymptotic3} \footnote{
    A technical remark on the application of \cite[Theorem 5.2 (iv)]{donsker1976asymptotic3}, which applies only over Banach spaces. When applying this result, we identify our probability measure $\mu$ on compact $\mathcal{X} \subset \R^{d}$ with its extension $\hat{\mu}$ on $\R^{d}$ defined by $\hat{\mu}(A) = \mu(A \cap \mathcal{X})$ for any Borel set $A$. Hence, we may apply \cite[Theorem 5.2 (iv)]{donsker1976asymptotic3} to find that $\kappa_{\hat{\mu}} =L^{*}_{\hat{\mu}}$. As integration with respect with $\mu$ or its extension $\hat{\mu}$ are identical, see, e.g., \cite[Example 16.4]{billingsley2017probability}, it follows $L_{\mu} = L_{\hat{\mu}}$, and in turn - with some minor proof details omitted - $\kappa_{\hat{\mu}}= \kappa_{\mu}$.\hfill$\diamond$} we have that
\begin{equation*}
    \kappa_{\mu}(x)  = \sup_{y \in \R^{d}} \left[  \langle y,x \rangle - \log \int_{\mathcal{X}} e^{\langle y,x\rangle} d\mu(x) \right](= L_{\mu}^{*}(x)).
\end{equation*}
Note that the conjugate $L_\mu^*$ is known in the large deviations literature as the  (Cram{\'e}r) rate function.

In summary, with our standing assumption that $\mathcal{X}$ is compact, $\kappa_{\mu} = L^{*}_{\mu}$. This directly implies the following properties of $\kappa_{\mu}$: (i) As $L_{\mu}$ is proper, lsc, and convex, so is its conjugate $L^{*}_{\mu} = \kappa_{\mu}$. (ii) Reiterating that $L_{\mu}$ is proper, lsc, convex, we may assert $(L_{\mu}^{*})^{*}= L_{\mu}$ via Fenchel-Moreau (\cite[Theorem 5.23]{royset2022optimization}), and hence $\kappa_{\mu}^{*} = L_{\mu}$. (iii) As $\dom(L_{\mu}) = \R^{d}$ we have that $\kappa_{\mu}$ is supercoercive \cite[Theorem 11.8 (d)]{rockafellar2009variational}. (iv) Recalling that $L_{\mu}$ is everywhere differentiable, $\kappa_{\mu}$ is strictly convex on every convex subset of $\dom (\partial \kappa_{\mu})$, which is also referred to as essentially strictly convex \cite[p.~253]{rockafellar1997convex}.

With these preliminary notions, we can (re-)state the problem of interest in full detail. We work with images represented as vectors in $\R^{d}$, where $d$ is the number of pixels. Given observed image $b \in \R^{m}$ which may be blurred and noisy, and known matrix $C \in \R^{m \times d}$, we wish to recover the ground truth $\hat{x}$ from the linear inverse problem $ b = C\hat{x} + \eta,$ where $\eta \sim  \mathcal{Z}$ is an unknown noise vector in $\R^{m}$ drawn from noise distribution $\mathcal{Z}$. We remark that, in practice, it is usually the case that $m=d$ and $C$ is invertible, but this is not necessary from a theoretical perspective. We assume the ground truth $\hat{x}$ is the expectation of an underlying image distribution - a Borel probability measure - $\mu$ on compact set $\mathcal{X} \subset \R^{d}$. Our best guess of $\hat{x}$ is then obtained by solving
\begin{equation*}
    \overline{x}_{\mu} = \argmin_{x \in \R^{d}} \alpha  g(Cx) +  \kappa_{\mu}(x).\tag{P} 
\end{equation*}
where $g = g_{b}$ is a proper, lsc, convex function which may depend on $b$ and serves as a fidelity term, and $\alpha >0$ a parameter. For example, if $g = \frac{1}{2}\Vert b - (\cdot) \Vert_{2}^{2}$ one recovers the  so-called reformulated MEM problem, first seen in \cite{le1999new}. 
\begin{lemma} \label{lemma:soln_exist} For any lsc, proper, convex $g$, the primal problem (P) always has a solution.
\end{lemma}

\begin{proof}  By the global assumption of compactness of $\mathcal{X}$, we have $\kappa_{\mu}$ is proper, lsc, convex and supercoercive, following the discussion above. As $g \circ C$ and $\kappa_{\mu}$ are convex, so is $g \circ C +\kappa_{\mu}$. Further as both $g \circ C$ and $\kappa_{\mu}$ are proper and lsc, and $\kappa_{\mu}$ is supercoercive, the summation $g \circ C +\kappa_{\mu}$ is supercoercive, \cite[Exercise 3.29, Lemma 3.27]{rockafellar2009variational}. A supercoercive function is, in particular, level-bounded, so by \cite[Theorem 1.9]{rockafellar2009variational} the solution set $\argmin( g \circ C +\kappa_{\mu}) $ is nonempty.
\end{proof}

We make one restriction on the choice of $g$, which will hold globally in this work:
\begin{assum}
     $0\in \inti(\dom(g) - C\dom(\kappa_{\mu}))$. \label{assum:domain}
\end{assum}
We remark that this property holds vacuously whenever $g$ is finite-valued, e.g., $g = \frac{1}{2}\Vert b - ( \cdot) \Vert^{2}_{2}$. 

Instead of solving (P) directly, we use a  dual approach. As $\kappa_\mu^*=L_\mu$ (by compactness of $\mathcal X$), the primal problem (P)
    has Fenchel dual  (e.g., \cite[Definition 15.19]{bauschke2019convex}) given by
    \begin{equation}
          ({\rm arg})\!\!\min_{z \in \R^{m}} \alpha g^{*}(-z/\alpha) + L_{\mu}(C^{T}z). \label{dual} \tag{D}
    \end{equation}
We will hereafter denote the dual objective  associated with $\mu \in \mathcal{P}(\mathcal{X})$ as
\begin{equation}
    \phi_{\mu}(z) := \alpha g^{*}(-z/\alpha) + L_{\mu}(C^{T}z).
\end{equation}
We record the following  result which highlights the significance of \Cref{assum:domain} to our study.

\begin{theorem} \label{thm:level-coercive} The following are equivalent:

\begin{center}
(i) \Cref{assum:domain} holds; \quad (ii) $\argmin \phi_\mu$ is nonempty and compact; \quad (iii) $\phi_\mu$ is level-coercive.
\end{center}

 \noindent
 In particular, under \Cref{assum:domain}, the primal problem (P) has a unique solution  given by 
\begin{equation}
    \overline{x}_{\mu} = \nabla L_{\mu}(C^{T}\overline{z}), \label{eqn:primal_dual_optimality}
\end{equation}
where $\overline{z} \in \argmin \phi_{\mu}$ is any solution of the dual problem (D).
\end{theorem}
\begin{proof} The equivalences follow from Proposition 3.1.3 and Theorem 5.2.1 in \cite{auslender2006interior}, respectively. The latter\footnote{Note that there is a sign error in equation (5.3) in the reference.} also yields the primal-dual recovery in \eqref{eqn:primal_dual_optimality} while using the differentiability of $L_\mu$.
\end{proof}

\subsection{Approximate and Empirical Priors, Random Functions, and Epi-consistency} \label{sec:approximation}

If one has access to the true underlying image distribution $\mu$, then the solution recipe is complete: solve (D) and and use the primal-dual recovery formula \cref{eqn:primal_dual_optimality} to find a solution to (P). But in practical situations, such as the imaging problems of interest here, it is unreasonable to assume full knowledge of $\mu$. Instead, one specifies a prior $\nu \in \mathcal{P}(\mathcal{X})$ with $\nu \approx \mu$, and solves the approximate dual problem
\begin{equation}
   \min_{z \in \R^{m}} \phi_{\nu}(z). \label{Dual_nu}
\end{equation}
Given $\varepsilon> 0$ and any $\varepsilon$-solution to \cref{Dual_nu}, i.e. given any $z_{\nu, \varepsilon} \in S_{\varepsilon} (\nu)$, we define
\begin{equation}
    \overline{x}_{\nu, \varepsilon} := \nabla L_{\nu}(C^{T}z_{\nu, \varepsilon}), \label{defn:x_nu}
\end{equation} 
with the hope, inspired by the recovery formula \cref{eqn:primal_dual_optimality}, that with a ``reasonable'' choice of $\nu \approx \mu$, and small $\varepsilon$, then also $\overline{x}_{\nu, \varepsilon} \approx \overline{x}_{\mu}$. The remainder of this work is dedicated to formalizing how well $\overline{x}_{\nu, \varepsilon}$ approximates $\overline{x}_{\mu}$ under various assumptions on $g$ and $\nu$.

A natural first approach is to construct $\nu$ from sample data. Let $(\Omega,\mathcal{F}, \Prob)$ be a probability space. We model image samples as i.i.d. $\mathcal{X}$-valued random variables $\{X_{1} , \ldots, X_{n}, \ldots \}$ with shared law $\mu := \Prob X_1^{-1}$. That is, each $X_{i} : \Omega \to \mathcal{X}$ is an $(\Omega, \mathcal{F}) \to (\mathcal{X}, \mathcal{B}_{\mathcal{X}})$ measurable function with the property that $\mu(B) = \Prob(\omega \in \Omega \: : \: X_1(\omega) \in B)$, for any $B \in \mathcal{B}_{\mathcal{X}}$. In particular, the law $\mu$ is by construction a Borel probability measure on $\mathcal{X}$. Intuitively, a random sample of $n$ images is a given sequence of realizations $\{ X_{1}(\omega), \ldots, X_{n}(\omega), \ldots \}$, from which we take only the first $n$ vectors. 
We then approximate $\mu$ via the empirical measure
\begin{equation*}
    \mu_{n}^{(\omega)} :=\frac{1}{n} \sum_{i=1}^{n} \delta_{X_{i}(\omega)}.
\end{equation*}
With this choice of $\nu = \mu_{n}^{(\omega)}$, we have the approximate dual problem
\begin{equation}
    \min_{z \in \R^{m}} \phi_{\mu_{n}^{(\omega)}}(z)  \quad {\rm with}\quad \phi_{\mu_{n}^{(\omega)}}(z)=\alpha g^{*}\left(\frac{-z}{\alpha}\right) + \log \frac{1}{n} \sum_{i=1}^{n} e^{\langle C^{T}z, X_{i}(\omega) \rangle }. \label{eqn:approx_dual} 
\end{equation}
And exactly analagous to \eqref{defn:x_nu}, given an $\varepsilon$-solution $\overline{z}_{n,\varepsilon}(\omega)$  of \cref{eqn:approx_dual},  we define
\begin{equation}
    \overline{x}_{n,\varepsilon}(\omega) := \nabla L_{\mu_{n}^{(\omega)}}(z)\vert_{z = \overline{z}_{n,\varepsilon}(\omega)} = \nabla_{z} \left[ \log \frac{1}{n} \sum_{i=1}^{n} e^{\langle C^{T}z, X_{i}(\omega) \rangle }
 \right]_{z =\overline{z}_{n,\varepsilon}(\omega)}. \label{defn:x_n}
\end{equation}

Clearly, while the measure $\mu_{n}^{(\omega)}$ is well-defined and Borel for any given $\omega$,  the convergence properties of $\overline{z}_{n, \varepsilon}(\omega)$ and $\overline{x}_{n, \epsilon}(\omega)$ should be studied in a stochastic sense over $\Omega$. To this end, we leverage a probabilistic version of epi-convergence for random functions known as epi-consistency \cite{king1991epi}. 

Let $(T, \mathcal{A})$ be a measurable space. A function $f : \R^{m} \times T \to \overline{\R}$ is called a random\footnote{The inclusion of the word `random' in this definition need not imply a priori any relation to a random process; we simply require measurability properties of $f$. Random lsc functions are also known as normal integrands in the literature, see \cite[Chapter 14]{rockafellar2009variational}.} lsc function (with respect to $(T,\mathcal{A})$)  \cite[Definition 8.50]{royset2022optimization} if the (set-valued) map $S_{f}: T  \rightrightarrows \R^{m+1}, \;S_{f}(t) = \epi f(\cdot, t)$
is closed-valued and measurable in the sense $S_{f}^{-1}(O) = \{ t \in T \: : \: S_{f}(x) \cap O \neq \emptyset \} \in \mathcal{A}$. 

Our study is fundamentally interested in random lsc functions on $(\Omega, \mathcal{F})$, in service of proving convergence results for $\overline{x}_{n, \epsilon}(\omega)$. But we emphasize that random lsc functions with respect to $(\Omega,\mathcal{F})$ are tightly linked with random lsc functions on $(X, \mathcal{B}_{\mathcal{X}})$. Specifically, if $X: \Omega \to \mathcal{X}$ is a random variable and $f: \R^{m} \times \mathcal{X} \to \overline{\R}$ is a random lsc function with respect to $(\mathcal{X}, \mathcal{B}_{\mathcal{X}})$, then the composition $f(\cdot, X(\cdot)): \R^{m} \times \Omega \to \R$ is a random lsc function with respect to the measurable space $(\Omega, \mathcal{F})$, see e.g. \cite[Proposition 14.45 (c)]{rockafellar2009variational}
 or the discussion of \cite[Section 5]{romisch2007stability}. This link will prove computationally convenient in the next section.

While the definition of a random lsc function is unwieldy to work with directly, it is implied by a host of easy to verify conditions \cite[Example 8.51]{royset2022optimization}.
We will foremost use the following one: Let $(T, \mathcal{A})$ be a measurable space. If a function $f:\R^{m} \times T \to \overline{\R}$ is finite valued, with $f(\cdot, t)$ continuous for all $t$, and $f(z, \cdot)$ measurable for all $z$, we say $f$ is a Carath{\'e}odory function. Any function which is Carath{\'e}odory is random lsc \cite[Example 14.26]{rockafellar2006variational}. 

Immediately, we can assert $\phi_{\mu_{n}^{(\cdot)}}$ is a random lsc function from $\R^{d} \times \Omega \to \overline{\R}$, as it is Carath{\'e}odory. In particular, by \cite[Theorem 14.37]{rockafellar2009variational} or \cite[Section 5]{romisch2007stability}, the $\varepsilon$-solution mappings
\begin{equation*}
    \omega \mapsto \left\{ z \: : \: \phi_{\mu_{n}^{(\omega)}}(z) \leq \inf \phi_{\mu_{n}^{(\omega)}} + \varepsilon \right\}
\end{equation*}
are measurable (in the set valued sense defined above), and it is always possible to find a $\overline{z}(\omega) \in \argmin \phi_{\mu_{n}^{(\omega)}}$ such that the function $\omega \mapsto \overline{z}(\omega)$ is $\Prob$-measurable in the usual sense.

We conclude with the definition of epi-consistency as seen in \cite[p.~86]{king1991epi}; a sequence of random lsc functions $h_{n}: \R^{m}
\times \Omega \to \overline{\R} $ is said to be epi-consistent with limit function $h: \R^{m} \to \overline{\R}$ if
\begin{equation}
\Prob\left(\left\{ \omega \in \Omega \: \vert \:  h_{n}(\cdot,\omega) \toe h \right\}\right) =1. \label{def:epi-consistent}
\end{equation}


\section{Epigraphical convergence and convergence of minimizers} \label{sec:epi-convergence}

The goal of this section is to prove convergence of minimizers in the empirical case, i.e., that $\overline{x}_{n,\varepsilon}(\omega)$ as defined in \eqref{defn:x_n} converges to $\overline{x}_{\mu}$, the solution of (P), for $\Prob$-almost every $\omega\in \Omega$ as $\varepsilon \downarrow 0$. To do so, we prove empirical approximations of the moment generating function are epi-consistent with $M_{\mu}$, and parley this into a proof of the epi-consistency of $\phi_{\mu_{n}^{(\omega)}}$ with limit $\phi_{\mu}$. Via classic convex analysis techniques, this guarantees the desired convergence of minimizers with probability one.

\subsection{Epi-consistency of the empirical moment generating functions}

Given $\{X_{1}, \ldots, X_{n}, \ldots\}$ i.i.d. with shared law $\mu = \Prob X_1^{-1} \in\mathcal P(\mathcal{X})$, we denote the moment generating function of $\mu_{n}^{(\omega)}$ as $M_{n}(y, \omega) :=  \frac{1}{n} \sum_{i=1}^{n} e^{\langle y, X_{i}(\omega) \rangle}.$ Define $f: \R^{m} \times \R^{d} \to \R$ as $f(z, x) =  e^{\langle C^{T}z, x\rangle}$. Then

\begin{align*}
    M_{\mu}(C^{T}z) &= \int_{\mathcal{X}} e^{\langle C^{T}z, \cdot \rangle} d \mu = \int_{\mathcal{X}} f(z, \cdot) d\mu, \\
    M_{n}(C^{T}z, \omega) & = \frac{1}{n} \sum_{i=1}^{n} e^{\langle C^{T}z, X_{i}(\omega) \rangle} = \frac{1}{n} \sum_{i=1}^{n} f(z, X_{i}(\omega)).
\end{align*}

\noindent
This explicit decomposition is useful to apply a specialized version of the main theorem of King and Wets \cite[Theorem 2]{king1991epi}, which we restate without proof.

\begin{prop} \label{thm:epicon}
Let $f : \R^{m} \times \mathcal{X} \to \overline{\R}$ be a random lsc function such that $f(\cdot, x)$ is convex and differentiable for all $x$. Let $X_{1}, \ldots, X_{n}$ be i.i.d. $\mathcal{X}$-valued random variables on $(\Omega, \mathcal{F}, \Prob)$ with shared law $\mu \in \mathcal{P}(\mathcal{X})$. If there exists $\overline{z} \in \R^{m}$ such that
    \begin{equation*}
        \int_{\mathcal{X}} f(\overline{z},\cdot) d\mu < +\infty, \qquad \text { and } \qquad
        \int_{\mathcal{X}} \Vert \nabla_{z}f(\overline{z}, \cdot) \Vert d\mu < + \infty,
    \end{equation*}
    then the sequence of (random lsc) functions $S_{n}: \mathbb{R}^{m} \times \Omega \to \overline{\R}$ given by
    \begin{equation*}
        S_{n}(z, \omega) := \frac{1}{n} \sum_{i=1}^{n}f(z, X_{i}(\omega))
    \end{equation*}
    is epi-consistent with limit $S_{\mu}:z\mapsto\int_{\mathcal{X}} f(z, \cdot) d\mu$, which is proper, convex, and lsc.
\end{prop}
\noindent
Via a direct application of the above we have the following.

\begin{corollary} \label{thm:epicon_mgf} The sequence  $M_{n}(C^{T}(\cdot), \cdot)$ is  epi-consistent with limit $M_{\mu} \circ C^{T}$.
\end{corollary}

\begin{proof}
    Define $f(z,x) = e^{\langle C^{T}z, x\rangle}$. For any $x$, $\langle C^{T} (\cdot),x \rangle$ is a linear function, and $e^{ ( \cdot ) }$ is convex - giving that the composition $f(\cdot, x)$ is convex. As $f$ is differentiable (hence continuous) in $z$ for fixed $x$ and vice-versa, it is Carath{\'e}odory and thus a random lsc function (with respect to $(\mathcal{X},\mathcal{B}_{\mathcal{X}})$). 
    
    Next we claim $\overline{z} = 0$ satisfies the conditions of the proposition. First, by direct computation
    \begin{equation*}
    \int_{\mathcal{X}} e^{ \langle 0,x \rangle } d\mu(x) = \int_{\mathcal{X}} d\mu(x) = 1 < + \infty
\end{equation*}
 as $\mu$ is a probability measure on $\mathcal{X}$. As $f(\cdot, x)$ is differentiable, we can compute $\nabla_{z}f(\overline{z},x) = Cxe^{\langle C^{T}z,x\rangle} \vert_{z = 0} =Cx$. Hence 
\begin{equation*}
    \int_{\mathcal{X}} \Vert \nabla_{z}f(\overline{z},x) \Vert d\mu(x) = \int_{\mathcal{X}} \Vert C x\Vert d\mu(x) \leq \Vert C \Vert \max_{x \in \mathcal{X}} \Vert x \Vert < + \infty,
\end{equation*} 
where we have used the boundedness of $\mathcal{X}$, and once again that $\mu$ is a probability measure.
Thus we satisfy the assumptions of \cref{thm:epicon}, and can conclude that the sequence of random lsc functions $S_{n}$ given by $S_{n}(z,\omega) = \frac{1}{n}\sum_{i=1}^{n} f(z, X_{i}(\omega))$ are epi-consistent with limit $S_{\mu} : z \mapsto \int_{\mathcal{X}} f(z , \cdot) d\mu$. But, 
\begin{equation*}
    S_{n}(z, \omega) = \frac{1}{n} \sum_{i=1}^{n} e^{\langle C^{T}z, X_{i}(\omega) \rangle} = M_{n}(C^{T} z, \omega) \qquad \text{ and } \qquad S_{\mu}(z) = \int_{\mathcal{X}} e^{\langle C^{T}z, \cdot \rangle} d\mu = M_{\mu}(C^{T}z),
\end{equation*}
and so we have shown the sequence $M_{n}(C^{T}(\cdot), \cdot)$ is epi-consistent with limit $M_{\mu} \circ C^{T}$.
\end{proof}

 \begin{corollary} \label{cor:Log_MGF_epiconverges} The sequence $L_{\mu_{n}^{(\omega)}} \circ C^{T}$ is epi-consistent with limit $L_{\mu} \circ C^{T}$.
 \end{corollary}
 \begin{proof}
 Let
 \begin{equation*}
      \Omega_{e} = \left\{ \omega \in \Omega \: \vert \: M_{n}(C^{T}(\cdot),\omega) \toe  M_{\mu} \circ C^{T}(\cdot) \right\},
 \end{equation*}
which has $\Prob(\Omega_{e})=1$ by \cref{thm:epicon_mgf}, and let $\omega \in \Omega_{e}$. Both $M_{n}$ and $M_{\mu}$ are finite valued and strictly positive, and furthermore the function $\log: \R_{++} \to \R$ is continuous and increasing. Hence,  by a simple extension of \cite[Exercise 7.8(c)]{rockafellar1997convex}, it follows, for all $\omega \in \Omega_{e}$, that
\[
    L_{\mu_{n}^{(\omega)}}\circ C^{T} = \log M_{n}(C^{T}(\cdot),\omega) \toe \log M_{\mu} \circ C^{T} = L_{\mu} \circ C^{T}.
\]
\hfill
\end{proof}

\subsection{Epi-consistency of the dual objective functions}

We now use the previous lemma to obtain an epi-consistency result for the entire empirical dual objective function. This is not an immediately clear, as epi-convergence is not generally preserved by even simple operations such as addition, see, e.g., the discussion in \cite[p.~276]{rockafellar2009variational} and the note \cite{BuH15} that eludes to subtle difficulties when dealing with extended real-valued arithmetic in this context. \\

We recall the following pointwise convergence result for compact $\mathcal{X}$, which is classical in the statistics literature. 
\begin{lemma}\label{lemma:MGF_pointwise}
    If $\mu \in \mathcal{P}(\mathcal{X})$, for almost every $\omega \in \Omega$, and all $z \in \R^{m}$
    \begin{equation*}
       M_{n}(C^{T}z, \omega) \to  M_{\mu} \circ C^{T}(z),
    \end{equation*}
    namely pointwise convergence in $z$.
\end{lemma}

We remark that the literature contains stronger uniform convergence results, observed first in Cs{\"o}rg{\"o} \cite{csorgo1982empirical} without proof, and later proven in \cite{feuerverger1989empirical} and \cite[Proposition 1]{csorgHo1983kernel}. Noting that both $M_{n}(z, \omega),M_{\mu}(z) > 0$ are strictly positive for all $z \in \R^{m}$, and that the logarithm is continuous on the strictly positive real line, we have an immediate corollary:

\begin{corollary} \label{cor:Logmgf_pointwise}
For almost every $\omega \in \Omega$, for all $z \in \R^{m}$
    \begin{equation*}
      L_{\mu_{n}^{(\omega)}}(C^{T}z)  = \log M_{n}(C^{T}z, \omega) \to \log  M_{\mu}(C^{T}z) = L_{\mu}( C^{T}z  ).
    \end{equation*}
\end{corollary}

Using this we prove the first main result:

\begin{theorem} \label{thm:epicon_dual_obj} For any lsc, proper, convex function $g$, the empirical dual objective function $\phi_{\mu_{n}^{(\omega)}}$ is epi-consistent with limit $\phi_{\mu}$    
\end{theorem}

\begin{proof}

Define 
\begin{equation*}
    \Omega_{e} = \left\{ \omega \in \Omega \: \vert \: L_{\mu_{n}^{(\omega)}}\circ C^{T}(\cdot) \toe L_{\mu} \circ C^{T}(\cdot)\right\}.
\end{equation*}
By \cref{cor:Log_MGF_epiconverges}, $\Prob(\Omega_{e})=1$. Similarly denote 
\begin{equation*}
    \Omega_{p} = \left\{ \omega \in \Omega \: \vert \: L_{\mu_{n}^{(\omega)}}\circ C^{T}(\cdot) \to  L_{\mu} \circ C^{T}(\cdot) \text{ pointwise} \right\}.
\end{equation*} 
By \cref{cor:Logmgf_pointwise}, we also have $\Prob(\Omega_{p})=1$. In particular we observe that $\Prob(\Omega_{e} \cap \Omega_{p})=1$.

On the other hand we have vacuously that the constant sequence of convex, proper, lsc functions $\alpha g^{*}\circ  (-\text{Id}/\alpha)$ converges to $\alpha g^{*}\circ ( - \text{Id}/\alpha)$ both epigraphically and pointwise. \\

Thus for any fixed $\omega \in \Omega_{p} \cap \Omega_{e}$ we have constructed two sequences, namely $g_{n} \equiv \alpha g^{*} \circ (-\text{Id}/\alpha)$ and $L_{n} = L_{\mu_{n}^{(\omega)}}\circ C^{T}$, which both converge epigraphically and pointwise for all $\omega \in \Omega_{e} \cap \Omega_{p}$. Therefore, by \cite[Theorem 7.46(a)]{rockafellar2009variational}, for all $\omega \in \Omega_{e} \cap \Omega_{p}$
\begin{equation*}
\alpha g^{*}\circ (- \text{Id}/\alpha) + L_{\mu_{n}^{(\omega)}}\circ C^{T} \toe \alpha g^{*}\circ (- \text{Id}/\alpha) + L_{\mu} \circ C^{T} .
\end{equation*}
As $\Prob(\Omega_{e} \cap \Omega_{p}) =1$, this proves the result.
\end{proof}

\subsection{Convergence of minimizers}

We now parley epi-consistency into convergence of minimizers. At the dual level this can be summarized in the following lemma, essentially \cite[Proposition 2.2]{king1991epi}; which was stated therein without proof.\footnote{
We remark that (as observed in \cite{king1991epi}) epigraphical convergence of a (multi-)function depending on a parameter (such as $\omega$) guarantees convergence of minimizers in much broader contexts, see e.g. \cite[Theorem 1.10]{attouch1984variational} or \cite[Theorem 3.22]{rockafellar2006variational}. Here we include a first principles proof.}

\begin{lemma} \label{lemma:min} There exists a subset $\Xi \subset \Omega$ of measure one, such that for any $\omega \in \Xi$ we have: Let $\{ \varepsilon_{n} \} \searrow 0$ and $z_{n}(\omega)$ such that
    \begin{equation*}
        \phi_{\mu_{n}^{(\omega)}}(z_{n}(\omega)) \leq \inf_{z} \phi_{\mu_{n}^{(\omega)}}(z) + \varepsilon_{n}.
    \end{equation*}
    Let $\{ z_{n_{k}}(\omega) \}$ be any convergent subsequence of $\{ z_{n}(\omega) \} $. Then $\lim_{k \to \infty}z_{n_{k}}(\omega)$ is a minimizer of $\phi_{\mu}$. If $\phi_{\mu}$ admits a unique minimizer $\overline{z}_{\mu}$, then $z_{n} \to \overline{z}_{\mu}$.
\end{lemma}

\begin{proof}
Denote 
\begin{equation*}
    \Xi = \left\{ \omega \in \Omega \: \vert \: \phi_{\mu_{n}^{(\omega)}} \toe \phi_{\mu} \right\}.
\end{equation*}
By \cref{thm:epicon_dual_obj}, $\Prob(\Xi) = 1$. Fix any $\omega \in \Xi$. 

By \Cref{thm:level-coercive}, the global \cref{assum:domain} holds if and only if $\phi_{\mu}$ is level-bounded. As $\omega \in \Xi$, we have that the sequence of convex functions $\phi_{\mu_{n}^{(\omega)}} \toe \phi_{\mu}$ epi-converges to a level-bounded function, and therefore by \cite[Theorem 7.32 (c)]{rockafellar2009variational}, the sequence $\phi_{\mu_{n}}^{(\omega)}$ is eventually level-bounded.\footnote{A sequence of functions $f_{n}: \R^{d} \to \overline{\R}$ is eventually level-bounded if for each $\alpha$, the sequence of sets $\{ f_{n}^{-1}([-\infty, \alpha])\}$ is eventually bounded, see \cite[p.~266]{rockafellar2009variational}.} 
In particular this means the sequence of lsc, proper, eventually level-bounded functions $\phi_{\mu_{n}^{(\omega)}}$ epi-converge to $\phi_{\mu}$, which is also lsc and proper. Hence by \cite[Theorem 7.33]{rockafellar2009variational} any sequence of approximate minimizers $\{ z_{n}(\omega) \}$ is bounded and with all cluster points belonging to $\argmin \phi_{\mu} $. Namely, any convergent subsequence $\{ z_{n_{k}}(\omega) \}$ has the property that its limit $\lim_{k \to \infty} z_{n_{k}}  \in \argmin \phi_{\mu} $. Lastly, if we also have $\argmin \phi_{\mu}  = \{ \overline{z}_{\mu} \}$, then from the same result \cite[Theorem 7.33]{rockafellar2009variational}, then necessarily $z_{n}(\omega) \to \overline{z}_{\mu}$.
\end{proof}
We now push this convergence to the primal level by using, in essence, Attouch's Theorem \cite{attouch1977convergence}, \cite[Theorem 3.66]{attouch1984variational}, in the form of a corollary of Rockafellar and Wets \cite[Theorem 12.40]{rockafellar2009variational}.

\begin{lemma} \label{lemma:gradient_converge}
    Let $\hat{z} \in \R^{m}$, and let $z_{n} \to \hat{z}$ be any sequence converging to $\hat{z}$. Then for almost every $\omega$,
    \begin{equation*}
        \lim_{n \to \infty} \nabla L_{\mu_{n}^{(\omega)}}(C^{T}z)\vert_{z = z_{n}} = \nabla L_{\mu}(C^{T}\hat{z}).
    \end{equation*}
\end{lemma}

\begin{proof} We first observe that $\dom (L_{\mu} \circ C^{T}) = \R^{m}$ so that $\hat{z} \in \text{int}(\text{dom}(L_{\mu} \circ C^{T} )).$ Also as $M_{\mu}$ is everywhere finite-valued,  $L_{\mu}(C^{T}\hat{z}) = \log M_{\mu}(C^{T}\hat{z}) < + \infty$. Furthermore for all $n$, the function $L_{\mu_{n}^{(\omega)}}\circ C^{T}$ is proper, convex, and differentiable. Finally, we have shown in \cref{cor:Log_MGF_epiconverges}, that for almost every $\omega \in \Omega$, we have $ L_{\mu_{n}^{(\omega)}}\circ C^{T} \toe L_{\mu} \circ C^{T}$. \\
These conditions together are the necessary assumptions of \cite[Theorem 12.40 (b)]{rockafellar2009variational}. Hence we have convergence $\lim_{n \to \infty} \nabla L_{\mu_{n}^{(\omega)}}(C^{T}z)\vert_{z = z_{n}} = \nabla L_{\mu}(C^{T}\hat{z})$ for almost every $\omega \in \Omega$.
\end{proof}

We now prove the main result.

\begin{theorem} \label{Thm:convergence_of_primal}
    There exists a set $\Xi \subseteq \Omega$ of probability one such that for each $\omega \in \Xi$ the following holds: Given $\varepsilon_{n} \searrow 0$, and $z_{n}(\omega)$ such that 
    $\phi_{\mu_{n}^{(\omega)}}(z_{n}(\omega)) \leq \inf_{z} \phi_{\mu_{n}^{(\omega)}}(z) + \varepsilon_{n}$, 
define 
\begin{equation*}
    x_{n}(\omega) := \nabla L_{\mu_{n}^{(\omega)}}(C^{T}z)\vert_{z = z_{n}}.
\end{equation*} 
If $z_{n_{k}}(\omega)$ is any convergent subsequence of $z_{n}(\omega)$ then $\lim_{k \to \infty} x_{n_{k}}(\omega) = \overline{x}_{\mu} $,
where $\overline{x}_{\mu}$ is the unique solution of $(P)$. If \cref{eqn:approx_dual} admits a unique solution $\overline{z}_{\mu}$, then in fact $x_{n}(\omega) \to \overline{x}_{\mu}.$
\end{theorem}

\begin{proof}
    Let 
    \begin{equation*}
        \Xi = \{ \omega \in \Omega \: \vert \:  \phi_{\mu_{n}^{(\omega)}} \toe \phi_{\mu} \},
    \end{equation*}
    recalling that by \cref{thm:epicon}, $\Prob(\Xi)=1$. Fix $\omega \in \Xi.$ By \cref{lemma:min}, for any convergent subsequence $z_{n_{k}}(\omega)$ with limit $\overline{z}(\omega)$, we have that $\overline{z}(\omega) \in \argmin \phi_{\mu}$. Furthermore, by  \cref{lemma:gradient_converge}
    \begin{equation*}
    \lim_{k \to \infty} x_{n_{k}}(\omega)  =\lim_{k \to \infty }\nabla L_{\mu_{n}^{(\omega)}}(C^{T}z)\vert_{z = z_{n}} = \nabla L_{\mu}(C^{T}\overline{z}(\omega))
    \end{equation*}
    Using the primal-dual optimality conditions \cref{eqn:primal_dual_optimality} we have that $\nabla L_{\mu}(C^{T}\overline{z}(\omega))$
    solves the primal problem (P). As (P) admits a unique solution $\overline{x}_{\mu}$, necessarily $\lim_{k \to \infty} x_{n_{k}}(\omega) = \overline{x}_{\mu}$. If additionally $\argmin \phi_{\mu} = \{\overline{z}_{\mu} \}$, then necessarily $z_{n} \to \overline{z}_{\mu}$ via \cref{lemma:min}, and the result follows from an identical application of \cref{lemma:gradient_converge} and \cref{eqn:primal_dual_optimality}.
\end{proof}

\section{Convergence rates for quadratic fidelity} \label{sec:rates}

When proving rates of convergence, we restrict ourselves to the case where $g = \frac{1}{2}\Vert b - (\cdot) \Vert_{2}^{2}$. Thus, the dual objective function reads
\begin{equation}\label{eq:Dual_2norm}
    \phi_{\mu}(z) = \frac{1}{2\alpha} \Vert z \Vert^{2} - \langle b, z \rangle + L_{\mu}(C^{T}z).
\end{equation}
Clearly, $\phi_{\mu}$ is  finite valued and ($1/\alpha$-)strongly convex, hence admits a unique minimizer $\overline{z}_{\mu}$.
Recalling what was laid out in \Cref{sec:MEMProblem}, as global \cref{assum:domain} holds vacuously with  $g = \frac{1}{2}\Vert c - ( \cdot ) \Vert_{2}^{2}$, the unique solution to the MEM primal problem (P) is given by $\overline{x}_{\mu} = \nabla L_{\mu}(C^{T}\overline{z}_{\mu})$. Further by our global compactness assumption on $\mathcal{X}$, $\phi_\mu$ is  (infinitely many times) differentiable.

\subsection{Epigraphical distances}
Our main workhorse to prove convergence rates are epigraphical distances. We mainly follow the presentation in Royset and Wets \cite[Chapter 6.J]{royset2022optimization}, but one may find similar treatment in Rockafellar and Wets \cite[Chapter 7]{rockafellar2009variational}. For any norm $\Vert \cdot \Vert_{*}$ on $\R^{d}$, the distance (in said norm) between a point $c$ and a set $D$ is defined as $ d_{D}(c) = \inf_{d \in D} \Vert c - d \Vert_{*} $. For $C,D$ subsets of $\R^{d}$ we define the excess of $C$ over $D$ \cite[p.~399]{royset2022optimization} as
\begin{equation*}
    \exc(C,D) := \begin{cases} \sup_{c \in C} d_{D}(c) &\text{ if } C,D \neq \emptyset, \\
    \infty &\text{ if } C \neq \emptyset, D = \emptyset, \\
    0 &\text{ otherwise.}
    \end{cases}
\end{equation*}
We note that this excess explicitly depends on the choice of norm used to define $d_{D}$. For the specific case of the 2-norm, we denote the projection of a point $a \in \R^{d}$ onto a closed, convex set $B \subset \R^{d}$ as the unique point $\text{proj}_{B}(a) \in B$ which achieves the minimum $\Vert \text{proj}_{B}(a) - a \Vert = \min_{b \in B} \Vert b - a \Vert.$ \\
The truncated $\rho$-Hausdorff distance \cite[p.~399]{royset2022optimization} between two sets $C,D \subset \R^{d}$ is defined as
\begin{equation*}
    \hat{d}_{\rho}(C,D) := \max \left\{ \exc(C \cap B_{\rho} ,D),\exc(D\cap B_{\rho} ,C)\right\},
\end{equation*}
where $B_{\rho} = \{ x \in \R^{d} \: : \: \Vert x \Vert_{*} \leq \rho \}$ is the closed ball of radius $\rho$ in $\R^{d}$. When discussing distances on $\R^{d}$, we will consistently make the choice of $\Vert \cdot \Vert_{*} = \Vert \cdot \Vert$  the $2$-norm. We note we can recover the usual Pompeiu-Hausdorff distance by taking $\rho \to \infty$.

However, to extend the truncated $\rho$-distance to epigraphs of functions - which here are subsets of $\R^{d+1}$ - we equip $\R^{d+1}$ with a very particular norm.
For any $z \in \R^{d+1},$ write $z = (x,a)$ for $x \in \R^{d},a\in \R$. Then for any $z_{1},z_{2} \in \R^{d+1}$ we define the norm
\begin{equation*}
   \Vert z_{1} - z_{2} \Vert_{*,d+1} = \Vert (x_{1},a_{1}) - (x_{2},a_{2}) \Vert_{*,d+1} := \max \{ \Vert x_{1}-x_{2} \Vert_{2}, \vert a_{1}-a_{2} \vert \}.
\end{equation*}
With this norm, we can define an epi-distance as in \cite[Equation 6.36]{royset2022optimization}: for  $f,h: \R^{d} \to \overline{\R}$ not identically $+\infty$, and $\rho > 0$ we define 
\begin{equation}
    \hat{d}_{\rho}(f,h) := \hat{d}_{\rho}( \epi f ,\epi h), \label{defn:epidistance}
\end{equation}
where $\R^{d+1}$ has been equipped with the norm $\Vert \cdot \Vert_{*,d+1}$. This epi-distance quantifies epigraphical convergence in the following sense: \cite[Theorem 7.58]{rockafellar2009variational}\footnote{We remark that while at first glance the definition of epi-distance seen in \cite[Theorem 7.58]{rockafellar2009variational} differs from ours (which agrees with \cite{royset2022optimization}), it is equivalent up to multiplication by a constant and rescaling in $\rho$. See \cite[Proposition 6.58]{royset2022optimization} and \cite[Proposition 7.61]{rockafellar2009variational} - the Kenmochi conditions - for details.} if $f$ is a proper function and $f_{n}$ a sequence of proper functions, then for any constant $\rho_{0} >0$:
\begin{equation*}
    f_{n} \toe f \text{ if and only if }  \hat{d}_{\rho}(f_{n},f) \to 0 \text{ for all $\rho> \rho_{0}$}.
\end{equation*}
\subsection{Convergence Rates}
We begin by proving rates of convergence for arbitrary prior $\nu$, and later specialize to the empirical case. To this end, we construct a key global constant $\rho_{0}$ induced by $C,b,\alpha$ in \cref{eq:Dual_2norm}: We define
\begin{equation}
    \rho_{0} := \max \left\{ \hat{\rho}, \frac{\hat{\rho}^{2}}{2\alpha} + \Vert b \Vert \hat{\rho} + \hat{\rho}\Vert C \Vert \vert \mathcal{X} \vert \right\}, \label{eqn:rho_0_defn}
\end{equation}
where $\hat{\rho} = 2\alpha( \Vert b \Vert + \Vert C \Vert \vert \mathcal{X} \vert )$ and $\vert \mathcal{X} \vert := \max_{x \in \mathcal{X}} \Vert x \Vert.$ We emphasize that our running compactness assumption on $\mathcal X$ is essential for  finiteness of $\rho_{0}$. The main feature of this constant is the following.

\begin{lemma} \label{lemma:rho_0_conditions} For any $\nu \in \mathcal{P}(\mathcal{X})$, let $\phi_{\nu}$ be the corresponding dual objective function as defined in \cref{eq:Dual_2norm}, which has a unique minimizer $\overline{z}_{\nu}$. Then $\rho_{0}$ has the following two properties:
\begin{equation*}
    (a) \quad \phi_{\nu}(\overline{z}_{\nu}) \in [-\rho_{0}, \rho_{0}], \qquad
    (b) \quad \Vert \overline{z}_{\nu} \Vert \leq \rho_{0}.
\end{equation*}
     
\end{lemma}
\begin{proof}
    We first claim that $\Vert \overline{z}_{\nu} \Vert \leq \hat{\rho}$. Let $z \in \R^{d}$ be such that $\Vert z \Vert > \hat{\rho}$. Then, 
    \begin{align*}
        \phi_{\nu}(z) &\geq \frac{\Vert z \Vert^{2}}{2 \alpha} - \Vert b \Vert \Vert z \Vert  + \log \int_{\mathcal{X}} \exp\left( -\Vert C \Vert \Vert x \Vert \Vert z \Vert\right) d\nu(x)
         \geq \Vert z \Vert \left(\frac{\Vert z \Vert}{2 \alpha} - \Vert b \Vert  -\Vert C \Vert \vert \mathcal{X} \vert \right).
    \end{align*}
   Here, in the first inequality we Cauchy-Schwarz (twice), and in the second we used that $\Vert x \Vert \leq \vert \mathcal{X} \vert$ and $\nu(\mathcal{X})=1$; hence the integral term can be bounded below by the constant $-\Vert z \Vert \Vert C \Vert \vert \mathcal{X} \vert$. From this estimate it is clear that $\Vert z \Vert >\hat{\rho}$ implies $\phi_{\nu}(z) > 0 $. But observing that $\phi_{\nu}(0)= 0,$ such $z$ cannot be a minimizer. Hence necessarily $\Vert \overline{z}_{\nu} \Vert \leq \hat{\rho} \leq \rho_{0}$. Once more, via Cauchy-Schwarz and using $\nu(\mathcal{X})=1$,  $\Vert \overline{z}_{\nu} \Vert \leq \hat{\rho} $,  we compute
    \begin{align*}
        \vert \phi_{\nu}(\overline{z}_{\nu}) \vert = \left\vert \frac{\Vert \overline{z}_{\nu} \Vert^{2}}{2 \alpha} - \langle b,\overline{z}_{\nu} \rangle + L_{\nu}(\overline{z}_{\nu})\right\vert &\leq \frac{ \hat{\rho}^{2}}{2 \alpha} + \hat{\rho} \Vert b \Vert + \log \int_{\mathcal{X}} \exp( \Vert C \Vert \vert \mathcal{X} \vert \hat{\rho} )d\nu(x)\\
        &= \frac{\hat{\rho}^{2}}{2\alpha} + \hat{\rho} \Vert b \Vert  + \hat{\rho}\Vert C \Vert \vert \mathcal{X}  \vert.
    \end{align*}
\end{proof}

\begin{lemma} \label{cor:epidistanceBoundedBySupnorm}
Let $\rho_{0}$ be given by \cref{eqn:rho_0_defn}. Then for all $\rho> \rho_{0}$ and all $\mu, \nu \in \mathcal{P}(\mathcal{X})$, we have
\begin{equation*}
     \hat{d}_{\rho}(\phi_{\mu},\phi_{\nu})  \leq \max_{z \in B_{\rho}} \vert L_{\nu}(C^{T}z) - L_{\mu}(C^{T}z) \vert.
\end{equation*}
\end{lemma}
\begin{proof} \cref{lemma:rho_0_conditions} guarantees that for both measures $\mu,\nu \in \mathcal{P}(\mathcal{X})$, we have
\begin{equation} 
\phi_{\nu}(\overline{z}_{\nu}),\phi_{\mu}(\overline{z}_{\mu}) \in [-\rho_{0}, \rho_{0}] \qquad \text{ and } \qquad \Vert \overline{z}_{\nu} \Vert,\Vert \overline{z}_{\mu} \Vert \leq \rho_{0}. \label{eqn:rho_0_conditions_both}
\end{equation}
These conditions imply, for any $\rho > \rho_{0}$, that the set $
    C_{\rho} := (\{ z : \phi_{\mu}(z) \leq \rho \} \cup \{  z : \phi_{\nu}(z) \leq \rho \} )\cap B_{\rho}$
is nonempty. This follows from \cref{eqn:rho_0_conditions_both} as for any $\rho > \rho_{0}$ the nonempty set $\{ \overline{z}_{\mu},\overline{z}_{\nu}\} \cap B_{\rho_{0}}$ is contained in  $ C_{\rho_{0}} \subset C_{\rho}.$
As $C_{\rho}$ is nonempty we may apply \cite[Theorem 6.59]{royset2022optimization} with $f = \phi_{\mu}$ and $g= \phi_{\nu}$ to obtain $\hat{d}_{\rho}(\phi_{\mu},\phi_{\nu})  \leq  \sup_{ z\in C_{\rho }} \vert \phi_{\nu}(z) - \phi_{\mu}(z) \vert.$
Then from the definition of $\phi_{\mu}$ and $\phi_{\nu}$, we have
\begin{align*}\
     \sup_{ z\in C_{\rho }} \vert \phi_{\nu}(z) - \phi_{\mu}(z) \vert
    &= \sup_{z \in C_{\rho}} \vert L_{\nu}(C^{T}z) - L_{\mu}(C^{T}z) \vert \\
    &\leq \sup_{z \in B_{\rho}} \vert L_{\nu}(C^{T}z) - L_{\mu}(C^{T}z) \vert \\ &= \max_{z \in B_{\rho}} \vert L_{\nu}(C^{T}z) - L_{\mu}(C^{T}z) \vert,
\end{align*}
where in the penultimate line uses that $C_{\rho} \subseteq B_{\rho}$, and the final equality follows as  the continuous function $L_{\mu}\circ C^{T} -L_{\nu} \circ C^{T}$ achieves a maximum  over the compact set $B_{\rho}$. 
\end{proof}

For notational convenience, we will hereafter denote
\begin{equation*}
     D_{\rho}(\nu,\mu) := \max_{z \in B_{\rho}} \vert L_{\mu}(C^{T}z) - L_{\nu}(C^{T}z) \vert.  
\end{equation*}

\noindent
We also recall from \Cref{sec:convexAnalysisPrelim} that $S_{\varepsilon}(\nu)$ denotes the set of $\varepsilon$-minimizers of $\phi_{\nu}$. 

\begin{lemma} \label{cor:infdistRoyset}
Let $\rho_{0}$ be given by \cref{eqn:rho_0_defn}. Then, for all $\mu, \nu \in \mathcal{P}(\mathcal{X})$, all $\rho > \rho_{0}$, and all $\varepsilon \in [0, \rho-\rho_{0}]$, the following holds: If
\begin{equation*}
    \delta > \varepsilon + 2   D_{\rho}(\nu,\mu)  ,
\end{equation*}
then
\begin{align*}
    \vert \phi_{\nu}(\overline{z}_{\nu}) - \phi_{\mu}(\overline{z}_{\mu}) \vert &\leq  D_{\rho}(\nu,\mu) \qquad \text{and} \qquad
     \exc(S_{\varepsilon}( \nu)\cap B_{\rho}, S_{\delta}(\mu) ) \leq D_{\rho}(\nu,\mu).
\end{align*}
    
\end{lemma}

\begin{proof}

Let $\rho > \rho_{0}$ and $\varepsilon \in [0,\rho-\rho_{0}]$. By choice, we have $\varepsilon < 2\rho$. By \cref{lemma:rho_0_conditions}(a) we have $\phi_{\nu}(\overline{z}_{\nu}), \phi_{\mu}(\overline{z}_{\mu}) \in [ -\rho_{0}, \rho_{0}],$
and in turn by the choice of $\rho, \varepsilon$ we have $ [ -\rho_{0}, \rho_{0}] \subseteq [ -\rho, \rho_{0}]  \subseteq [ -\rho, \rho - \varepsilon]$. Also, as $\rho > \rho_{0}$, by \cref{lemma:rho_0_conditions}(b) we have $\{ z_{\nu} \} = \argmin \phi_{\nu} \cap B_{\rho}$  and $\{ z_{\mu} \}  = \argmin \phi_{\mu} \cap B_{\rho}.$ These properties of $\rho, \varepsilon$ are exactly the  assumptions of \cite[Theorem 6.56]{royset2022optimization} for $f= \phi_{\mu}$ and $g= \phi_{\nu}$. This result yields that, if $\delta > \varepsilon +2 \hat{d}_{\rho}(\phi_{\mu},  \phi_{\nu} )$, then $\vert \phi_{\nu}(\overline{z}_{\nu}) - \phi_{\mu}(\overline{z}_{\mu}) \vert  \leq \hat{d}_{\rho}( \phi_{\nu}, \phi_{\mu} )$ and $\text{exs}(S_{\varepsilon}( \nu)\cap B_{\rho}, S_{\delta}(\mu) )\leq \hat{d}_{\rho}( \phi_{\nu}, \phi_{\mu} )$.\\
However as $ \rho > \rho_{0}$, we may apply \cref{cor:epidistanceBoundedBySupnorm} to assert $\hat{d}_{\rho}(\phi_{\mu},\phi_{\nu}) \leq  D_{\rho}(\nu,\mu)$. Hence, for any  $\delta > \varepsilon +2D_{\rho}(\nu,\mu) \geq \varepsilon +2 \hat{d}_{\rho}(\phi_{\mu},  \phi_{\nu} )$ we obtain
\begin{align*}
     \vert \phi_{\nu}(\overline{z}_{\nu}) - \phi_{\mu}(\overline{z}_{\mu}) \vert  & \leq D_{\rho}(\nu,\mu),\\
     \text{exs}(S_{\varepsilon}( \nu)\cap B_{\rho}, S_{\delta}(\mu) )& \leq D_{\rho}(\nu,\mu).
\end{align*}
\end{proof}

For the  main results, \cref{thm:epsdeltaprimalbound_full} and \cref{thm:final_rate_n}, we require additional auxiliary results.

\begin{lemma} \label{lemma:MGF_bounded} Let $\rho >0$ and $\nu \in \mathcal{P}(\mathcal{X})$. Then, for all $z \in B_{\rho}$ we have
\begin{equation*}
    M_{\nu}(C^{T}z) = \int_{\mathcal{X}} e^{\langle C^{T}z, \cdot \rangle} d\nu \in \left[\exp \left(  -\rho \Vert C \Vert \vert \mathcal{X} \vert \right), \exp \left(  \rho \Vert C \Vert \vert \mathcal{X} \vert \right) \right].
\end{equation*}
    
\end{lemma}
\begin{proof}
For all $x \in \mathcal{X}$, $z \in B_{\rho}$, we have, via Cauchy-Schwarz, that
\begin{equation*}
    \exp \left( - \rho \Vert C \Vert \vert \mathcal{X} \vert \right) \leq \exp \left( - \Vert z \Vert \Vert Cx \Vert  \right) \leq \exp \langle C^{T}z ,x \rangle.
\end{equation*}
In particular, $ \exp \left( - \rho \Vert C \Vert \vert \mathcal{X} \vert \right) \leq \min_{x \in \mathcal{X}} \exp \langle C^{T}z ,x \rangle.$ On the other hand, we find that
\begin{equation*}
    \exp \langle C^{T}z ,x \rangle \leq  \exp \left( \Vert z \Vert \Vert Cx \Vert \right) \leq \exp \left(  \rho \Vert C \Vert \vert \mathcal{X} \vert \right).
\end{equation*}
Thus, $\max_{x \in \mathcal{X}} \exp \langle C^{T}z ,x \rangle  \leq \exp \left(  \rho \Vert C \Vert \vert \mathcal{X} \vert \right).$ Hence for any $\nu \in \mathcal{P}(\mathcal{X})$ we find
\[
     1 \cdot\exp \left( - \rho \Vert C \Vert \vert \mathcal{X} \vert \right)\leq  \nu(\mathcal{X}) \min_{x \in \mathcal{X}} e^{\langle C^{T}z, x \rangle }  \leq \int_{\mathcal{X}}  e^{\langle C^{T}z, \cdot \rangle } d\nu \leq \nu(\mathcal{X}) \max_{x \in \mathcal{X}} e^{\langle C^{T}z, x \rangle } \leq 1 \cdot \exp \left(  \rho \Vert C \Vert \vert \mathcal{X} \vert \right).
 \]


\end{proof}

\noindent
With some additional computation, we can infer  the following Lipschitz bound on  $\nabla L_{\nu}$.

\begin{corollary}
    \label{cor:global_K_bound}
    Let $\hat{\rho} > 0$ and $\nu \in \mathcal{P}(\mathcal{X})$. Then for all $x,y \in B_{\hat{\rho}} \subset \R^{d}$, we have that 
    \begin{equation}
        \Vert \nabla L_{\nu}(x) - \nabla L_{\nu}(y) \Vert \leq K \Vert x-y \Vert 
    \end{equation}
    for an explicit constant $K>0$ which depends on $\hat{\rho}, d, \vert \mathcal{X} \vert$, but not on $\nu$.
\end{corollary} 

\begin{proof} 
    As discussed in \Cref{sec:MEMProblem}, $L_{\nu}$ is twice continuously differentiable. Hence, using the fundamental theorem of calculus, we have $\nabla L_{\nu}(x) - \nabla L_{\nu}(y) = \int_{0}^{1} \nabla^{2} L_{\mu}(x +t(y-x)) \cdot (y-x) dt.$ Thus, as $x +t(y-x) \in B_{\hat{\rho}}$ for all $t\in [0,1]$, we have
    \begin{align*}
        \Vert \nabla L_{\nu}(z) - \nabla L_{\nu}(y) \Vert &\leq \int_{0}^{1} \Vert \nabla^{2} L_{\mu}(x +t(y-x)) \Vert \Vert y-x \Vert dt \\
        &\leq \int_{0}^{1} \max_{z \in B_{\hat{\rho}}} \Vert \nabla^{2}L_{\nu}(z) \Vert \Vert y-x \Vert dt \\
        &= \max_{z \in B_{\hat{\rho}}} \Vert \nabla^{2}L_{\nu}(z) \Vert \Vert y-x \Vert.
    \end{align*}
    By convexity of $L_\nu$, we observe that  $\nabla^2 L_\nu(z)$ is (symmetric) positive semidefinite (for any $z$). Hence,  $\max_{z \in B_{\hat{\rho}}} \Vert \nabla^{2}L_{\nu}(z) \Vert  \leq \max_{z \in B_{\hat{\rho}}} \mathrm{Tr}(\nabla^{2} L_{\nu}(z) ).$
   Now, observe that
    \begin{align*}
        \frac{\partial^{2}}{\partial z_{i}^{2}} L_{\nu}(z) = \frac{\partial^{2}}{\partial z_{i}^{2}} \log M_{\nu}(z) &= \frac{-1}{(M_{\nu}(z))^{2}} \left[ \int_{\mathcal{X}} x_{i} \exp \langle z, x \rangle d\nu(x) \right]^{2}  + \frac{1}{M_{\nu}(z)} \left[ \int_{\mathcal{X}} x_{i}^{2} \exp \langle z, x \rangle d\nu(x) \right],
    \end{align*}
    where the interchange of the derivative and integral is permitted by the Leibniz rule for finite measures, see e.g. \cite[Theorem 2.27]{folland1999real} or \cite[Theorem 6.28]{klenke2013probability}. Taking the absolute value in the last identity, we may bound $\vert x_{i} \vert \leq \Vert x \Vert_{2} \leq \vert \mathcal{X} \vert$, $\Vert z \Vert \leq \hat{\rho}$, and apply \cref{lemma:MGF_bounded}  to bound $M_\nu(z)$. This  eventually yields
    \begin{align*}
        \left\vert \frac{\partial^{2}}{\partial z_{i}^{2}} L_{\nu}(z) \right\vert \leq 
        \frac{ \vert \mathcal{X} \vert  }{\exp(-\hat{\rho} \vert \mathcal{X} \vert)^{2}} \exp( \hat{\rho}  \vert \mathcal{X} \vert )^{2} + \frac{\vert \mathcal{X} \vert^{2} }{\exp(-\hat{\rho} \vert \mathcal{X} \vert)}  \exp( \hat{\rho} \vert \mathcal{X} \vert )=:\hat{K},
    \end{align*}
    with $\hat{K}>0$ which depends on $\hat \rho$ and $\vert \mathcal{X} \vert$. As this uniformly bounds every term in the trace, $K := d \cdot\hat{K}$ is the desired constant.
\end{proof}
The key feature of the constant $K$ is that it does not depend on the choice of measure $\nu$. Hence we can uniformly apply this bound over a family of measures, the most pertinent example being $\left\{ \mu_{n}^{(\omega)} \right\}$. We remark that our upper bound on $K$ is a vast overestimate for practical examples, which can be observed numerically. Finally we require a very simple technical lemma, whose proof is left as an exercise for brevity.

\begin{lemma} \label{lemma:excessDistanceBound}
    Let $A,B\subset \R^{d}$ be nonempty and let $B$ be closed and convex. Then for  $\overline{a} \in A$ and  $\overline{b} = \text{proj}_{B}(\overline{a})$ we have
    \begin{equation*}
        \Vert \overline{a} - \overline{b} \Vert \leq \exc(A;B).
    \end{equation*}
\end{lemma}

\noindent

We now have developed all the necessary tools to state and prove the main result for the case of $g = \frac{1}{2} \Vert (\cdot) -b \Vert$.

\begin{theorem} \label{thm:epsdeltaprimalbound_full}
Let $\rho_{0}$ be given by \cref{eqn:rho_0_defn}, and suppose $\mathrm{rank}(C)=d$. Then for all $\mu, \nu \in \mathcal{P}(\mathcal{X})$, all $\rho > \rho_{0}$ and all $\varepsilon \in [0, \rho -\rho_{0}]$, we have the following:
If $\overline{z}_{\nu,\varepsilon}$ is an $\varepsilon$-minimizer of $\phi_{\nu}$ as defined in \cref{eq:Dual_2norm}, then
    \begin{equation*}
        \overline{x}_{\nu, \varepsilon} := \nabla L_{\nu}(C^{T}\overline{z}_{\nu,\varepsilon})
    \end{equation*} 
    satisfies the error bound
\begin{align*}
    \left\Vert \overline{x}_{\nu,\varepsilon} -\overline{x}_{\mu}  \right\Vert \leq \frac{1}{\alpha \sigma_{\min}(C)}  D_{\rho}(\nu,\mu)  + \frac{2\sqrt{2}}{ \sqrt{\alpha} \sigma_{\min}(C)} \sqrt{ D_{\rho}(\nu,\mu) } + \left( K \Vert C \Vert \sqrt{2 \alpha } +\frac{2}{ \sqrt{\alpha} \sigma_{\min}(C)} \right) \sqrt{\varepsilon},
\end{align*}
 where $\overline{x}_{\mu}$ is the unique solution to the MEM primal problem $(P)$ for $\mu$ and $K>0$ is a constant which does not depend on $\mu, \nu$. 
\end{theorem}

\begin{proof}
Let $\rho > \rho_{0}$, $\nu, \mu \in \mathcal{P}(\mathcal{X})$ and $\varepsilon \in [0,\rho - \rho_{0}]$. Let $\overline{z}_{\nu,\varepsilon}$ be  a $\varepsilon$-minimizer of $\phi_{\nu}$, and denote the unique minimizers of $\phi_{\mu}$ and $\phi_{\mu}$ as $\overline{z}_{\nu}$ and $\overline{z}_{\mu}$, respectively. Then
\begin{align}
     \left\Vert \overline{x}_{\nu, \varepsilon} -\overline{x}_{\mu}  \right\Vert &= \left\Vert \nabla L_{\nu}(C^{T}\overline{z}_{\nu,\varepsilon})-  \nabla L_{\mu}(C^{T}\overline{z}_{\mu}) \right\Vert \nonumber \\
     &=\left\Vert \nabla L_{\nu}(C^{T}\overline{z}_{\nu,\varepsilon}) - \nabla L_{\nu}(C^{T}\overline{z}_{\nu}) + \nabla L_{\nu}(C^{T}\overline{z}_{\nu})-  \nabla L_{\mu}(C^{T}\overline{z}_{\mu}) \right\Vert, \nonumber
     \end{align}
and so
     \begin{align}
\left\Vert \overline{x}_{\nu, \varepsilon} -\overline{x}_{\mu}  \right\Vert &\leq \left\Vert \nabla L_{\nu}(C^{T}\overline{z}_{\nu,\varepsilon}) - \nabla L_{\nu}(C^{T}\overline{z}_{\nu}) \right\Vert + \left\Vert \nabla L_{\nu}(C^{T}\overline{z}_{\nu})-  \nabla L_{\mu}(C^{T}\overline{z}_{\mu}) \right\Vert \label{eq:term1+2}.
\end{align}
To estimate the first term on the right hand side of \cref{eq:term1+2}, we require an auxiliary bound. Observe that, as $\phi_{\nu}$ is strongly $1/\alpha$-convex with $\nabla \phi_{\nu}(\overline{z}_{\nu}) =0$, we have
\begin{align}
    \Vert \overline{z}_{\nu} - \overline{z}_{\nu, \varepsilon}\Vert \leq 
\sqrt{2\alpha } \vert  \phi_{\nu}(\overline{z}_{\nu}) -  \phi_{\nu}(\overline{z}_{\nu, \varepsilon}) \vert^{1/2} \leq  \sqrt{2 \alpha \varepsilon}. \label{eqn:epsilondistApproxMin}
\end{align}
Here the first inequality uses \cref{eqn:alternate_strongconvexity}, while the second follows from the definition of $\overline{z}_{\nu}$ and $\overline{z}_{\nu, \varepsilon}$, as $\vert  \phi_{\nu}(z_{\nu}) -  \phi_{\nu}(z_{\nu, \varepsilon}) \vert = \phi_{\nu}(z_{\nu, \varepsilon}) - \phi_{\nu}(z_{\nu}) \leq \varepsilon.$

From \cref{lemma:rho_0_conditions}(b), we find that $ \Vert \overline{z}_{\nu}  \Vert \leq \rho_{0}$. Thus,  \cref{eqn:epsilondistApproxMin} yields $\Vert \overline{z}_{\nu,\varepsilon} \Vert \leq \rho_{0} +\sqrt{2\alpha\varepsilon}$. This  implies  $\Vert C^{T}\overline{z}_{\nu} \Vert, \Vert C^{T}\overline{z}_{\nu, \varepsilon} \Vert \leq \Vert C \Vert (\rho_{0} +\sqrt{2\alpha \varepsilon})$.  Hence,  \cref{cor:global_K_bound} with  $\hat{\rho} =\Vert C \Vert (\rho_{0} +\sqrt{2\alpha \varepsilon})$  yields
\begin{equation*}
    \left\Vert \nabla L_{\nu}(C^{T}\overline{z}_{\nu,\varepsilon}) - \nabla L_{\nu}(C^{T}\overline{z}_{\nu}) \right\Vert \leq  K \Vert C^{T} \overline{z}_{\nu} - C^{T}\overline{z}_{\nu,\varepsilon} \Vert,
\end{equation*}
where $K$ depends on $\hat{\rho}, \vert \mathcal{X} \vert, d$ and therefore on $\vert \mathcal{X} \vert, \Vert C \Vert,b, \varepsilon, \alpha,d$. The right-hand side in the last inequality can be further estimated with   \cref{eqn:epsilondistApproxMin} to find 
\begin{equation}
    \Vert C^{T} \overline{z}_{\nu} - C^{T}\overline{z}_{\nu,\varepsilon} \Vert \leq \Vert C \Vert \Vert \overline{z}_{\nu} - \overline{z}_{\nu,\varepsilon} \Vert \leq \Vert C \Vert \sqrt{2\alpha \varepsilon}. \label{eq:final_left}
\end{equation}

We now turn to the second term on the right-hand side of \cref{eq:term1+2}. First order optimality conditions give
\begin{align*}
    0 = -\frac{\overline{z}_{\nu}}{\alpha} + b +C  \nabla  L_{\nu}(C^{T}\overline{z}_{\nu}) ,
   \qquad  0 = -\frac{\overline{z}_{\mu}}{\alpha} + b +  C \nabla L_{\mu}(C^{T}\overline{z}_{\mu}),
\end{align*}
and therefore $\left\Vert C( \nabla L_{\nu}(C^{T}\overline{z}_{\nu})-  \nabla L_{\mu}(C^{T}\overline{z}_{\mu})) \right\Vert = \frac{1}{\alpha} \Vert \overline{z}_{\nu} - \overline{z}_{\mu} \Vert$. Furthermore, as $\text{rank}(C)=d$ we have $\sigma_{\min}(C)>0$. We also have for, any $x \in \R^{d}$, that $ \Vert Cx \Vert \geq \sigma_{\min}(C) \Vert x \Vert$, and hence 
\begin{equation*}
    \left\Vert  \nabla L_{\nu}(C^{T}\overline{z}_{\nu})-  \nabla L_{\mu}(C^{T}\overline{z}_{\mu}) \right\Vert \leq \frac{1}{\sigma_{\min}(C)} \left\Vert  C(\nabla L_{\nu}(C^{T}\overline{z}_{\nu})-  \nabla L_{\mu}(C^{T}\overline{z}_{\mu})) \right\Vert =  \frac{1}{\alpha \sigma_{\min}(C)} \Vert \overline{z}_{\nu} - \overline{z}_{\mu} \Vert.
\end{equation*}

\noindent
In order to bound $\Vert \overline{z}_{\nu} - \overline{z}_{\mu} \Vert$ from above, we define $\delta :=  2 (\varepsilon + 2   D_{\rho}(\nu,\mu)  ).$ Denoting as usual $S_{\delta}(\mu)$ as the set of $\delta$-minimizers of $\phi_{\mu}$, which is a closed, convex set by the continuity and convexity of $\phi_{\mu}$ respectively, define $y = \text{proj}_{S_{\delta}(\mu)}(\overline{z}_{\nu}).$ The triangle inequality gives
\begin{equation}
    \Vert \overline{z}_{\nu} - \overline{z}_{\mu} \Vert \leq \Vert \overline{z}_{\nu} - y \Vert + \Vert y - \overline{z}_{\mu} \Vert. \label{eqn:projectionSplit}
\end{equation}
By the choice of $\rho > \rho_{0}$, we have by \cref{lemma:rho_0_conditions}(b) that $\overline{z}_{\nu}\in S_{\varepsilon}(\nu) \cap B_{\rho}$. Therefore applying \cref{lemma:excessDistanceBound} with $A = S_{\varepsilon}(\nu) \cap B_{\rho}$, $B = S_{\delta}(\mu)$ we can bound the first term on the right hand side of 
\cref{eqn:projectionSplit} as
\begin{equation}
    \Vert \overline{z}_{\nu} - y \Vert \leq \text{exc}(S_{\varepsilon}(\nu) \cap B_{\rho} ; S_{\delta}(\mu)). \label{eq:finalright_1}
\end{equation}
For the remaining term of the right hand side of 
\cref{eqn:projectionSplit}, we use the characterization \cref{eqn:alternate_strongconvexity} of the $\frac{1}{\alpha}$-strong convexity in the differentiable case for $\phi_{\mu}$, noting $\nabla \phi_{\mu}(\overline{z}_{\mu}) =0$. Hence
\begin{equation}
    \Vert y - \overline{z}_{\mu} \Vert \leq \sqrt{2\alpha} \vert \phi_{\mu}(y) - \phi_{\mu}(\overline{z}_{\mu}) \vert^{1/2} \leq \sqrt{2\alpha \delta}, \label{eq:finalright_2}
\end{equation}
where  $y \in S_{\delta}(\mu)$ for the second inequality.
Combining \cref{eq:final_left,eqn:projectionSplit,eq:finalright_1,eq:finalright_2} with  \cref{eq:term1+2}, we find that
\begin{equation*}
    \left\Vert \overline{x}_{\nu, \varepsilon} -\overline{x}_{\mu}  \right\Vert \leq  K \Vert C \Vert \sqrt{2 \alpha \varepsilon} + \frac{1}{\alpha \sigma_{\min}(C)} \text{exc}(S_{\varepsilon}(\nu) \cap B_{\rho}; S_{\delta}(\mu))   + \frac{1}{\alpha \sigma_{\min}(C)} \sqrt{2\alpha \delta}.
\end{equation*}
By the choice of $\delta = 2 (\varepsilon + 2   D_{\rho}(\nu,\mu)  )$, \cref{cor:infdistRoyset} asserts $\text{exc}(S_{\varepsilon}(\nu) \cap B_{\rho}; S_{\delta}(\mu)) \leq D_{\rho}(\nu,\mu)  .$ Therefore
\begin{align*}
\left\Vert \overline{x}_{\nu, \varepsilon} -\overline{x}_{\mu}  \right\Vert &\leq \frac{1}{\alpha \sigma_{\min}(C)} \text{exc}(S_{\varepsilon}(\nu) \cap B_{\rho}; S_{\delta}(\mu))   + \frac{1}{\alpha \sigma_{\min}(C)} \sqrt{2\alpha \delta} + K \Vert C \Vert \sqrt{2 \alpha \varepsilon}\\
&\leq \frac{1}{\alpha \sigma_{\min}(C)} D_{\rho}(\nu,\mu)  + \frac{1}{\sigma_{\min}(C)}\sqrt{\frac{ 4 \varepsilon}{\alpha}} + \frac{1}{\sigma_{\min}(C)}\sqrt{ \frac{ 8  D_{\rho}(\nu,\mu) }{\alpha}} + K \Vert C \Vert \sqrt{2 \alpha \varepsilon} \\
&= \frac{1}{\alpha \sigma_{\min}(C)}  D_{\rho}(\nu,\mu)  + \frac{2 \sqrt{2}}{ \sqrt{\alpha} \sigma_{\min}(C)} \sqrt{ D_{\rho}(\nu,\mu) } + \left( K \Vert C \Vert \sqrt{2 \alpha } +\frac{2}{ \sqrt{\alpha} \sigma_{\min}(C)} \right) \sqrt{\varepsilon}
\end{align*}
where in the second line we have used the definition of $\delta$ and the concavity of $\sqrt{x+y} \leq \sqrt{x} + \sqrt{y}$. 
\end{proof}

\noindent
Note that we may set $\varepsilon =0$ for a corollary on exact minimizers. However, the error bound still has the same scaling in terms of $D_{\rho}(\nu,\mu)$.

\section[A statistical dependence on n]{A statistical dependence on $\bm{n}$} \label{sec:rates_n_empirical}

This section is devoted to making the dependence on $n$ explicit in \cref{thm:epsdeltaprimalbound_full} for the special case $\nu = \mu_{n}^{(\omega)}$. We briefly recall the empirical setting   developed in \Cref{sec:approximation}. Given i.i.d. random vectors $\{ X_{1}, X_{2}, \ldots , X_{n}, \ldots\} $ on $(\Omega,\mathcal{F}, \Prob)$ with shared law $\mu = \Prob X_{1}^{-1}$, we define $\mu_{n}^{(\omega)} = \sum_{i=1}^{n} \delta_{X_{i}(\omega)}$. For this measure, the dual objective reads
\begin{align*}
    \phi_{\mu_{n}^{(\omega)}}(z) &= \frac{1}{2\alpha} \Vert z \Vert^{2} - \langle b, z \rangle + \log \frac{1}{n} \sum_{i=1}^{n} e^{\langle C^{T}z, X_{i}(\omega)\rangle}.
\end{align*}
Given $\overline{z}_{n, \varepsilon}(\omega)$,  an $\varepsilon$-minimizer of $\phi_{\mu_{n}^{(\omega)}}(z)$, define
\begin{equation*}
    \overline{x}_{n,\varepsilon}(\omega): = \nabla_{z} \left[ \log \frac{1}{n} \sum_{i=1}^{n} e^{\langle C^{T}z, X_{i}(\omega) \rangle }
 \right]_{z =\overline{z}_{n, \varepsilon}(\omega)}.
\end{equation*}

\noindent
We begin with a simplifying lemma, recalling the notation developed in \Cref{sec:rates} of the moment generating function $M_{\mu}$ of $\mu$ and empirical moment generating function $M_{n}(\cdot, \omega)$ of $\mu_{n}^{(\omega)}$.

\begin{lemma} \label{lemma:LogmgfboundedbyMGF}
    Let  $\rho > 0$, $n \in \mathbb{N}$ and  $\omega \in \Omega$ and set $K:= \exp \left(  \rho \Vert C \Vert \vert \mathcal{X} \vert \right)$. Then
    \begin{equation*}
        D_{\rho}(\mu,\mu_{n}^{(\omega)})   \leq K \max_{z \in B_{\rho}} \left\vert  M_{\mu}(C^{T}z) - M_{n}(C^{T}z,\omega) \right\vert.
    \end{equation*}
\end{lemma}

\begin{proof}
Applying \cref{lemma:MGF_bounded} to the particular probability measures $\mu$ and $\mu_{n}^{(\omega)}$ gives
\begin{equation}
    M_{\mu}(C^{T}z), M_{n}(C^{T}z,\omega) \in \left[\exp \left(  -\rho \Vert C \Vert \vert \mathcal{X} \vert \right), \exp \left(  \rho \Vert C \Vert \vert \mathcal{X} \vert \right) \right] =: [c,d] \label{eqn:upper_lowerbound_MGF}
\end{equation}
where $0< c< d.$ Furthermore, for any $s,t \in [c,d]$ we have $\vert \log(s) - \log(t) \vert \leq \frac{1}{c} \vert s - t \vert$, and hence
\[
    D_{\rho}(\mu,\mu_{n}^{(\omega)}) = \max_{z \in B_{\rho}} \vert L_{\mu_{n}^{(\omega)}}(C^{T}z) - L_{\mu}(C^{T}z) \vert  \leq \exp \left(  \rho \Vert C \Vert \vert \mathcal{X} \vert \right) \max_{z \in B_{\rho}} \left\vert  M_{\mu}(C^{T}z) - M_{n}(C^{T}z,\omega) \right\vert.
\]

\end{proof}
\begin{lemma} \label{lemma:primal_bound_n_global_constants}
   Let $\rho_{0}$ be as defined in \cref{eqn:rho_0_defn} and suppose $\mathrm{rank}(C)=d$. Then, for all $\rho > \rho_{0}$, all $\varepsilon \in [0, \rho-\rho_{0}]$, and for all $n \in \mathbb{N}$ we have: if $\overline{z}_{n,\varepsilon}(\omega) \in S_{\varepsilon}(\mu_{n}^{(\omega)})$, then $\overline{x}_{n, \varepsilon}(\omega) = \nabla L_{\mu_{n}^{(\omega)}}(C^{T}\overline{z}_{n,\varepsilon})$ satisfies 
\begin{align*}
    \left\Vert \overline{x}_{n,\varepsilon}(\omega) -\overline{x}_{\mu}  \right\Vert &\leq \frac{K_{1}}{\alpha \sigma_{\min}(C)}  \max_{z \in B_{\rho}} \left\vert  M_{\mu}(C^{T}z) - M_{n}(C^{T}z,\omega) \right\vert  \\
    &+ \frac{2 \sqrt{2 K_{1}}}{\sqrt{\alpha} \sigma_{\min}(C) } \sqrt{ \max_{z \in B_{\rho}} \left\vert  M_{\mu}(C^{T}z) - M_{n}(C^{T}z,\omega) \right\vert } 
    + \left( K_{2} \Vert C \Vert \sqrt{2 \alpha  } +\frac{2}{\sqrt{\alpha}\sigma_{\min}(C)} \right) \sqrt{\varepsilon}
\end{align*}
 where $K_{1}$ is a constant which depends on $\rho, \vert \mathcal{X} \vert, \Vert C \Vert$, and $K_{2}$ on $\vert \mathcal{X} \vert, \Vert C \Vert,b, \varepsilon, d, \alpha$.
\end{lemma}

\begin{proof}
 As $\mu_{n}^{(\omega)} \in \mathcal{P}(\mathcal{X})$, \cref{thm:epsdeltaprimalbound_full} yields for all $n$, for $\rho_{0}$ as defined in \cref{eqn:rho_0_defn}, for all $\rho> \rho_{0}$, and all $\varepsilon \in [0, \rho-\rho_{0}]$,  if $\overline{z}_{n,\varepsilon}(\omega) \in S_{\varepsilon}(\mu_{n}^{(\omega)})$, then $\overline{x}_{n, \varepsilon}(\omega) = \nabla L_{\nu}(C^{T}z_{n,\varepsilon})$ satisfies
 \begin{align*}
     \left\Vert \overline{x}_{n,\varepsilon}(\omega) -\overline{x}_{\mu}  \right\Vert &\leq \frac{1}{\alpha \sigma_{\min}(C)}  D_{\rho}(\mu,\mu_{n}^{(\omega)})  + \frac{2 \sqrt{2}}{\sqrt{\alpha} \sigma_{\min}(C)} \sqrt{ D_{\rho}(\mu,\mu_{n}^{(\omega)}) } \\
     &+ \left( K_{2} \Vert C \Vert \sqrt{2 \alpha } +\frac{2}{\sqrt{\alpha} \sigma_{\min}(C)} \right) \sqrt{\varepsilon},
 \end{align*}
where we stress that the constant $K_{2}$ depends on $\vert \mathcal{X} \vert, \Vert C \Vert,b, \varepsilon, \alpha,d$, but does not depend on $n$. Applying \cref{lemma:LogmgfboundedbyMGF} to bound $D_{\rho}(\mu,\mu_{n}^{(\omega)})   \leq K_{1} \sup_{z \in B_{\rho}} \left\vert  M_{\mu}(C^{T}z) - M_{n}(C^{T}z,\omega) \right\vert$ gives the result.
\end{proof}

\noindent
In order to construct a final bound which depends explicitly on $n$ it remains to estimate the term
$\max_{z \in B_{\rho}} \left\vert  M_{\mu}(C^{T}z) - M_{n}(C^{T}z,\omega) \right\vert$. This fits into the language of empirical process theory where this type of convergence is  well studied. The main reference of interest here is Van Der Vaart \cite{van1996new}.

For compact $\mathcal{X} \subset \R^{d}$, let  $f: \mathcal{X} \to \R$ be a function and $\beta = (\beta_{1}, \ldots, \beta_{d})$ be a multi-index, i.e. a vector of $d$ nonnegative integers. We call $\vert \beta \vert = \sum_{i} \beta_{i}$ the order of $\beta$, and define the differential operator $ D^{\beta} =  \frac{\partial^{\beta_{1}}}{\partial x_{1}^{\beta_{1}}} \cdots \frac{\partial^{\beta_{d}}}{\partial x_{d}^{\beta_{d}}}.$ For integer $k$, we denote by $\mathcal C^{k}(\mathcal X)$ as the space of $k$-smooth (also known as $k$-H\"older continuous) functions on $\mathcal{X}$, namely, those $f$ which satisfy \cite[p.~2131]{van1996new}
\[
    \|f\|_{\mathcal C^{k}(\mathcal X)} :=  \max_{|\beta|\leq k}\sup_{x\in\text{int}(\mathcal X)}\|D^{\beta}f(x)\|+\max_{|\beta|=k}\sup_{\substack{x,y\in \text{int} (\mathcal X)\\x\neq y}}\left|\frac{D^{\beta}f(x)-D^{\beta}f(y)}{\|x-y\|}\right|<\infty.
\]
\noindent
    Moreover, let $\mathcal C^{k}_R(\mathcal X)$ denote the ball of radius $R$ in $\mathcal C^{k}(\mathcal X)$. With this notation developed we can state the classical results of Van der Vaart \cite{van1996new}. In the notation therein, we apply the machinery of Sections 1 and 2 to the measure space $(\mathcal{X}_{1}, \mathcal{A}_{1})$ = $(\mathcal{X}, \mathcal{B}_{\mathcal{X}})$, equipped with probability measure $\mu$. Taking $\mathbb{G}_{n} = \sqrt{n}(\mu_{n}^{(\omega)} - \mu)$ and $\mathcal{F}_{1} = \mathcal{F} = C^{k}_{R}( \mathcal{X})$, this induces the norm $\Vert \mathbb{G}_{n} \Vert_{\mathcal{F}} = \sup_{f \in C^{k}_R(\mathcal X)} \{ \left\vert \int_\mathcal{X} f d\mathbb{G}_{n}  \right\vert  \}$, and hence the results of \cite[p.~2131]{van1996new} give

\begin{theorem} \label{thm:donsker}
    Let $\mu \in \mathcal{P}(\mathcal{X})$. If $k > d/2$, then for any $R>0$, 
    \begin{equation*}
    \E_{\Prob}^{*} \left[ \sup_{f \in C^{k}_{R}(\mathcal X) } \sqrt{n} \left\vert \ \int_{\mathcal{X}} fd\mu_{n}^{(\cdot)} - \int_{\mathcal{X}}fd\mu \right\vert\right] \leq D
    \end{equation*}
    where $D$ is a constant depending (polynomially) on $k,d,\vert \mathcal{X} \vert$, and $\E_{\Prob}^{*}$ is the outer expectation to avoid concerns of measurablity (see e.g. \cite[Section 1.2]{van1996weak}).
\end{theorem}

We remark that outer expectation is also known as the outer integral \cite[Chapter 14.F]{rockafellar2009variational}, and coincides with the usual expectation for measurable functions - which are the only functions of interest in this work. A self-contained proof of the above result is non-trivial, requiring the development of entropy and bracketing numbers of function spaces which is beyond the scope of this article.\footnote{Bounds of this ``Donsker'' type have previously been applied to empirical approximations of stochastic optimization problems, to derive large deviation-style results for specific problems, see \cite[Section 5.5]{romisch2007stability} for a detailed exposition and discussion, in particular \cite[Theorem 5.2]{romisch2007stability}. In principle this machinery could be used here to derive similar large deviation results. } We simply take this result as given. However, we show the following corollary.

\begin{corollary} \label{cor:BoundedMGF}
    For all $n \in \mathbb{N}$, we have
    \begin{equation*}
         \E_{\Prob}\left[ \max_{z \in B_{\rho}} \left\vert  M_{\mu}(C^{T}z) - M_{n}(C^{T}z,\cdot) \right\vert  \right] \leq \frac{D}{\sqrt{n}},
    \end{equation*}
    where $D$ is a constant depending on $d,\vert \mathcal{X} \vert$.
\end{corollary}

\begin{proof}
    Observe that for each $z$, the function $f_{z}(x) = \exp \langle C^{T}z, \cdot \rangle$ is an infinitely differentiable function on the compact set $\mathcal{X}$, and thus has bounded derivatives of all orders, in particular, of order $k=d > d/2$. Hence, for  $\rho > 0$, the set of functions $f_{z}$ parameterized by $z\in B_{\rho}$ satisfies
    \[
        \left\{ f_{z}(x) = \exp\langle C^{\top}z,x\rangle:z\in B_{\rho} \right\}\subset \mathcal C^{d}_{R_{d}}(\mathcal X),
    \]
    where $R_{d}$ is a constant which depends on $d,\rho,\|C\|$, and $\vert \mathcal{X} \vert$. Furthermore, as  $\mathbb{Q}^{m} \cap B_{\rho} \subset B_{\rho}$ is a countable dense subset, $\max_{z \in B_{\rho}} \left\vert  M_{\mu}(C^{T}z) - M_{n}(C^{T}z,\cdot) \right\vert = \sup_{z \in \mathbb{Q}^{m} \cap B_{\rho}} \left\vert  M_{\mu}(C^{T}z) - M_{n}(C^{T}z,\cdot) \right\vert$ is a supremem of countably many $\Prob$-measurable functions and is hence $\Prob$-measurable. In particular the usual expectation agrees with the outer expectation. Hence applying \cref{thm:donsker} we may assert
 \begin{align*}
    \E_{\Prob}\left[ \max_{z \in B_{\rho}} \left\vert  M_{\mu}(C^{T}z) - M_{n}(C^{T}z,\cdot) \right\vert  \right] &= \E_{\Prob}^{*} \left[ \sup_{z \in B_{\rho}} \left\vert \ \int_{\mathcal{X}} f_{z}d\mu_{n}^{(\cdot)} - \int_{\mathcal{X}}f_{z}\mu \right\vert\right] \\
    &\leq \E_{\Prob}^{*} \left[ \sup_{f \in C^{d}_{R_{d}}(\mathcal X) } \left\vert \ \int_{\mathcal{X}} fd\mu_{n}^{(\cdot)} - \int_{\mathcal{X}}f\mu \right\vert\right] \\
    &\leq \frac{D}{\sqrt{n}},
 \end{align*}
 for a constant $D$ which depends on $d, \vert \mathcal{X} \vert$. We remark that the choice of $k=d$ in the above was aesthetic, so that the the constant $D$ only depends on $d, \vert \mathcal{X} \vert.$ 
\end{proof}
The final result now follows as a simple consequence of \cref{cor:BoundedMGF} and \cref{lemma:primal_bound_n_global_constants}:

\begin{theorem} \label{thm:final_rate_n}
Suppose $\mathrm{rank}(C)=d$. For all $n \in \mathbb{N}$, and all $\overline{z}_{n,\varepsilon}(\omega) \in S_{\varepsilon}(\mu_{n}^{(\omega)})$, the associated $\overline{x}_{n, \varepsilon}(\omega) = \nabla L_{\mu_{n}^{(\omega)}}(C^{T}\overline{z}_{n,\varepsilon})$ satisfies
\begin{align*}
    \E_{\Prob} \left\Vert \overline{x}_{n,\varepsilon}(\cdot) -\overline{x}_{\mu}  \right\Vert &\leq \frac{D K_{1}}{\alpha \sigma_{\min}(C)} \frac{1}{\sqrt{n}} + \frac{2D   \sqrt{2 K_{1}}}{\sqrt{\alpha} \sigma_{\min}(C) } \sqrt{ \frac{1}{\sqrt{n}} } + \left( K_{2} \Vert C \Vert \sqrt{2 \alpha } +\frac{2}{\sqrt{\alpha} \sigma_{\min}(C)} \right) \sqrt{\varepsilon} \\
    &= O\left( \frac{1}{n^{1/4}} + \sqrt{\varepsilon} \right),
\end{align*}
 where the leading constants $K_{1},K_{2},D$ depend on $\vert \mathcal{X} \vert, \Vert C \Vert,b, \varepsilon, \alpha,d$.
 \end{theorem}

 \begin{proof}
     Take $\rho = 2 \rho_{0}$ in \cref{lemma:primal_bound_n_global_constants}. This choice of $\rho$ gives that for all $n$ and all $\overline{z}_{n,\varepsilon}(
     \omega) \in S_{\varepsilon}(\mu_{n}^{(\omega)})$, the error bound
     \begin{align*}
    \left\Vert \overline{x}_{n,\varepsilon}(\omega) -\overline{x}_{\mu}  \right\Vert& \leq \frac{K_{1}}{\alpha \sigma_{\min}(C)}  \sup_{z \in B_{\rho}} \left\vert  M_{\mu}(C^{T}z) - M_{n}(C^{T}z,\omega) \right\vert\\
    & \hspace{-5mm}+ \frac{2 \sqrt{2} \sqrt{K_{1}}}{\sqrt{\alpha} \sigma_{\min}(C) } \sqrt{ \sup_{z \in B_{\rho}} \left\vert  M_{\mu}(C^{T}z) - M_{n}(C^{T}z,\omega) \right\vert } + \left( K_{2} \Vert C \Vert \sqrt{2 \alpha } +\frac{2}{\sqrt{\alpha} \sigma_{\min}(C)} \right) \sqrt{\varepsilon},
    \end{align*}
     holds with constants $K_{1}, K_{2}$ that depend only on $\vert \mathcal{X} \vert, \Vert C \Vert,b, \varepsilon, \alpha,d$. As $\omega \to \overline{x}_{n, \varepsilon}(\omega)$ is the composition of the continuous function $\nabla L_{\mu_{n}^{(\omega)}}$ and the measurable function  $C^{T}z_{n}(\omega)$ (see the discussion at the end of \Cref{sec:approximation}) the left hand side is $\Prob$-measurable. Hence taking the $\Prob$-expectation on both sides and applying \cref{thm:donsker} gives the result.
 \end{proof}

\section{Numerical experiments} \label{sec:numerics}

We now shift to a numerical examination of the convergence $\overline{x}_{n,\varepsilon} \to \overline{x}_{\mu}$. We focus entirely on the most recent setting of \Cref{sec:rates_n_empirical}, the MEM problem with an empirical prior $\mu_{n}^{(\omega)}$ and 2-norm fidelity. As discussed throughout, the empirical dual objective function $\phi_{\mu_{n}}^{(\omega)}$ is smooth and strongly convex, with easily computable derivatives\footnote{The derivatives are easily avaliable analytically, but a na{\"i}ve implementation may run into stablity issues. This can be easily addressed, see e.g. \cite{Kantas2015Particle}.}. Hence we may solve this problem with any off-the-shelf optimization package, and here we choose to use  limited memory BFGS \cite{byrd1994representations}.
As a companion to this paper, we include a jupyter notebook, \href{https://github.com/mattkingros/MEM-Denoising-and-Deblurring.git}{https://github.com/mattkingros/MEM-Denoising-and-Deblurring.git}, which can be used to reproduce and extend all computations performed hereafter.

For our numerical experiments we will focus purely on denoising, namely the case $C=I$. We will use two datasets and two distributions of noise as proof of concept. For noise, we use additive Gaussian noise and ``salt-and-pepper" corruption noise where each pixel has an independent probability of being set to purely black or white. For datasets, the first is the MNIST digits dataset \cite{deng2012mnist} which contains 60000 28x28 grayscale pixel images of hand-drawn digits 0-9, and the second is the more expressive dataset of Fashion-MNIST \cite{xiao2017fashion}, which is once again 60000 28x28 grayscale pixel images but of various garments such as shoes, sneakers, bags, pants, and so on. 

In all experiments we \textit{always} include enough noise so that the nearest neighbour to $b$ in the dataset belongs to a different class than the ground truth, ensuring that recovery is non-trivial. We include this in all figures, captioned ``closest point''. Furthermore, we take the ground truth to be a new data-point hand drawn by the authors; in particular, it is {\it not present} in either of the original MNIST or MNIST fashion datasets.

We begin with a baseline examination of the error in $n$ for the MNIST digit dataset, for various choices of $b$. To generate $\mu_{n}^{(\omega)}$ practically, we sample $n$ datapoints uniformly at random without replacement. As a remark on this methodology, the target best possible approximation is the one which uses all possible data, i.e., $\mu_{D} := \mu_{60,000}^{(\omega)}$. Hence, for error plots we compare to $\overline{x}_{\mu_{D}}$, as we do not have access to the full image distribution to construct $\overline{x}_{\mu}$. 

Given a noisy image $b$, the experimental setup is as follows: for 20 values of $n$, spaced linearly between $10000$ and $60000$, we perform $15$ random samples of size $n$. For each random sample, we compute an approximation $\overline{x}_{n, \epsilon}$ and the relative approximation error $\frac{\Vert \overline{x}_{n, \epsilon} - \overline{x}_{\mu_{D}} \Vert}{\Vert \overline{x}_{\mu_{D}} \Vert}$, which is then averaged over the trials. Superimposed is the upper bound $Kn^{1/4}$ convergence rate, for moderate constant $K$ which changes between figures. We also visually exhibit several reconstructions, the nearest neighbour, and postprocessed images for comparison. This methodology is used to create figures \cref{fig:Gaus_noise_8,fig:SP_noise_8,fig:Gaus_noise_shirt,fig:SP_noise_shirt,fig:gaus_noise_heel,fig:SP_noise_heel}

A remark on postprocessing: As alluded to in the introduction (see \cref{remark:measure_valued}), an advantage of the dual approach is the ability to reconstruct the optimal measure $Q_{n}$ which solves the measure-valued primal problem as seen in \cite{rioux2020maximum}. In particular as $Q_{n} \ll \mu_{n}^{(\omega)}$, we have $\overline{x}_{n} = \E_{Q_{n}}$ is a particular weighted linear combination of the input data. This allows for two types of natural postprocessing: at the level of the linear combination, i.e., setting all weights below some given threshold to zero, or at the level of the pixels: i.e. setting all pixels above 1-$\gamma$ to 1 and below $\gamma$ to 0. Hence our figures also include the final measure $Q_{n}$ with all entries below $0.01$ set to zero, the corresponding linear combination of the remaining datapoints (bottom right image), and a further masking at the pixel level (top right image).

With this methodology developed, we present the results. For the first figures, the ground truth is a hand-drawn $8$, which is approximated by sampling the MNIST digit dataset. For \cref{fig:Gaus_noise_8} we use additive Gaussian noise of variance $\sigma = 0.10\Vert x\Vert $. For \cref{fig:SP_noise_8} we use instead use salt-and-pepper corruptions with an equal probability of 0.2 of any given pixel being set to 1 or 0.

\begin{figure}[h]
\centering
\begin{tabular}{cccc}
    \includegraphics[width = 0.17\textwidth]{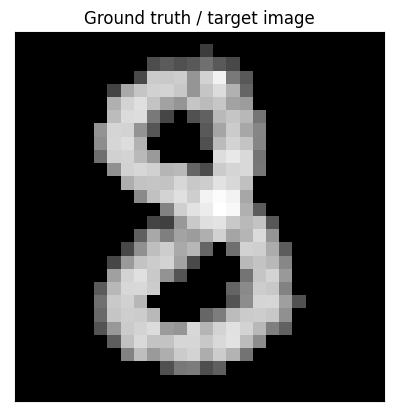} &
    \includegraphics[width = 0.17\textwidth]{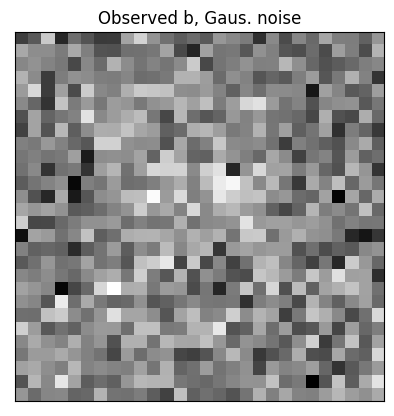} & 
    \includegraphics[width = 0.17\textwidth]{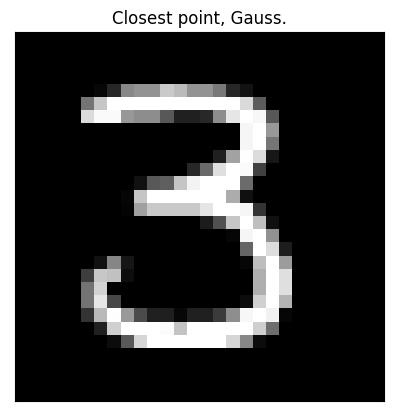} &
    \includegraphics[width = 0.17\textwidth]{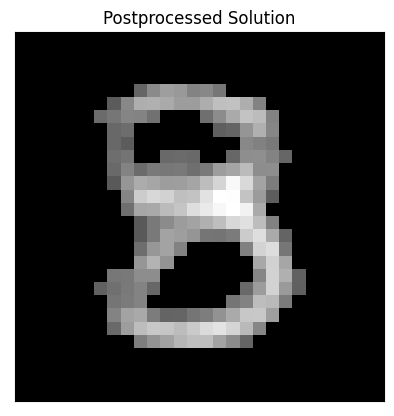} \\ 
    \includegraphics[width = 0.2\textwidth]{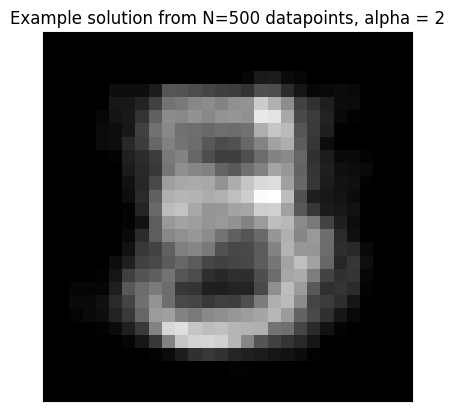} &
     \includegraphics[width = 0.2\textwidth]{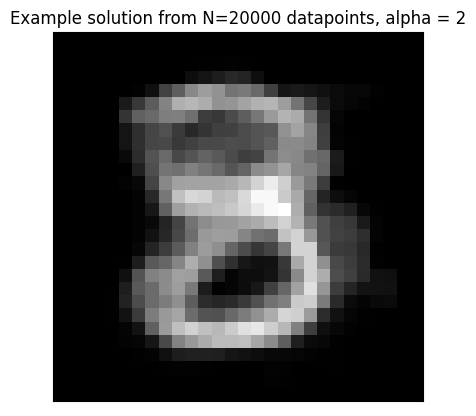} &
     \includegraphics[width = 0.2\textwidth]{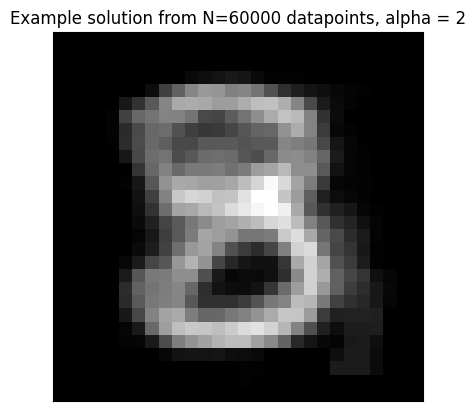} &
     \includegraphics[width = 0.17\textwidth]{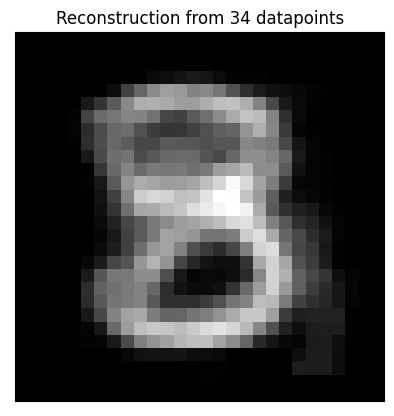}
\end{tabular}
    \caption{Recovery of an 8 with additive Gaussian noise}
    \label{fig:Gaus_noise_8}
\end{figure}

\begin{figure}[h]
\centering
\begin{tabular}{cc}
    \includegraphics[width = 0.4\textwidth]{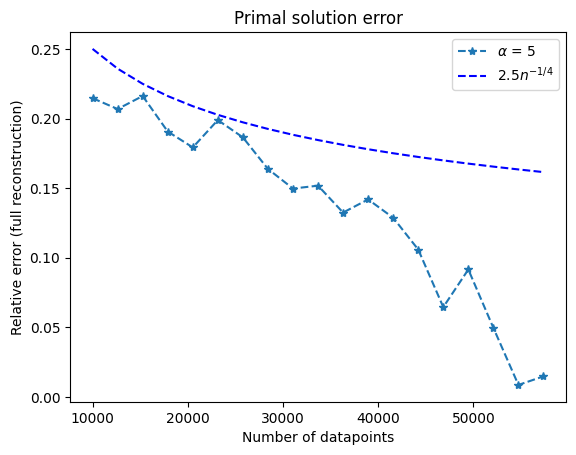} &
    \includegraphics[width = 0.4\textwidth]{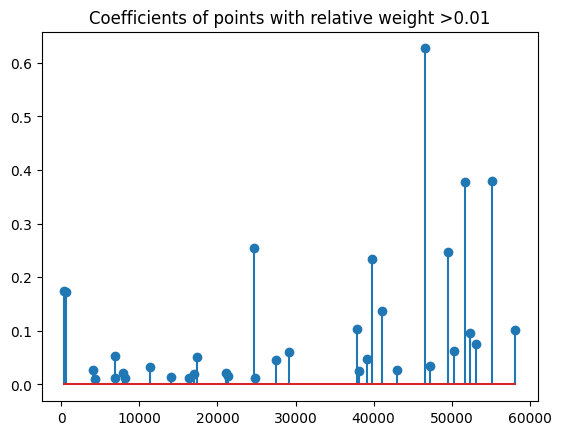} 
\end{tabular}
\caption{Rates and thresholding of optimal measure $Q_{n}$, for \cref{fig:Gaus_noise_8} }
\label{fig:Gaus_noise_8_rates}
\end{figure}

\begin{figure}[h]
\centering
\begin{tabular}{cccc}
    \includegraphics[width = 0.17\textwidth]{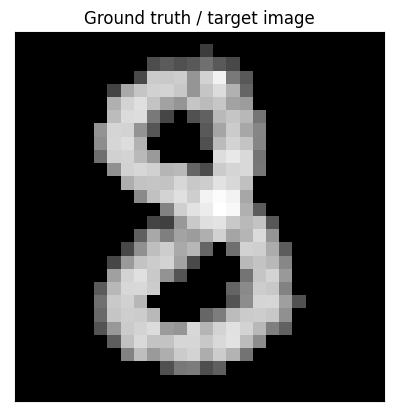} &
    \includegraphics[width = 0.17\textwidth]{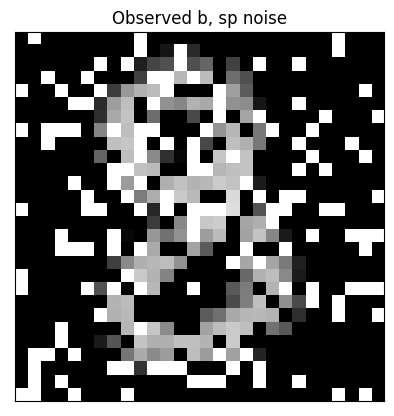} & 
    \includegraphics[width = 0.17\textwidth]{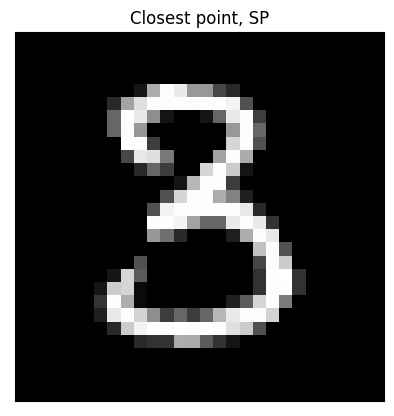} &
    \includegraphics[width = 0.17\textwidth]{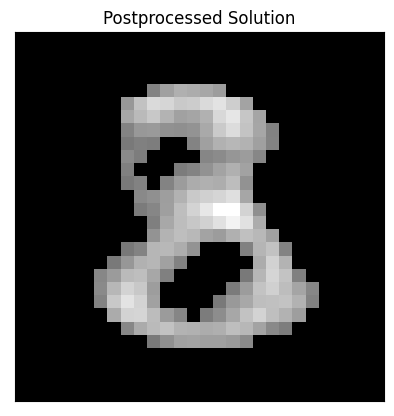}  \\ 
    \includegraphics[width = 0.20\textwidth]{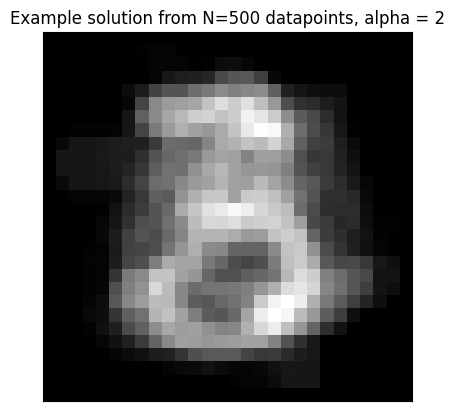} &
     \includegraphics[width = 0.20\textwidth]{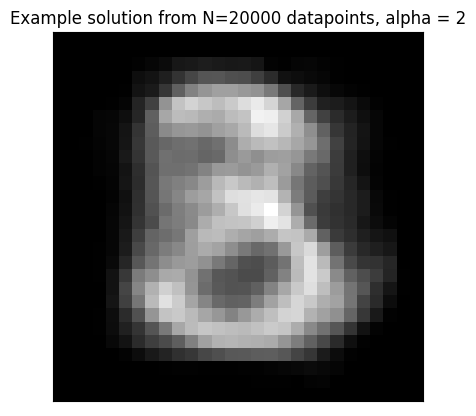} &
     \includegraphics[width = 0.20\textwidth]{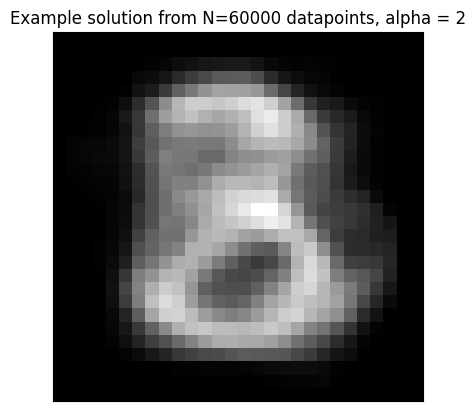} &
     \includegraphics[width = 0.17\textwidth]{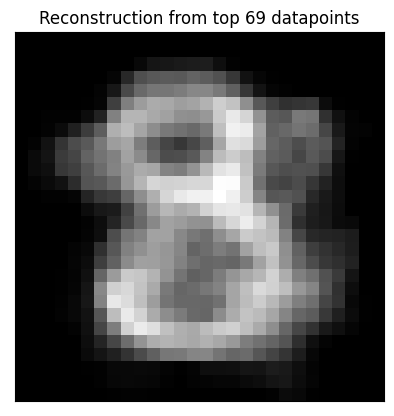} 
\end{tabular}
    \caption{Recovery of an 8 with salt-and-pepper corruption noise}
    \label{fig:SP_noise_8}
\end{figure}

\begin{figure}[h]
\centering
\begin{tabular}{cc}
    \includegraphics[width = 0.4\textwidth]{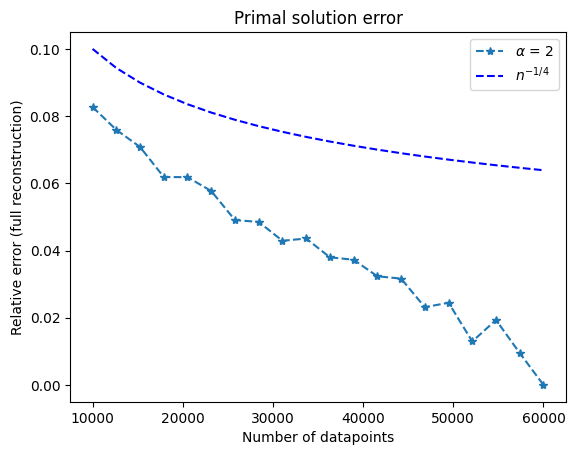} &
    \includegraphics[width = 0.4\textwidth]{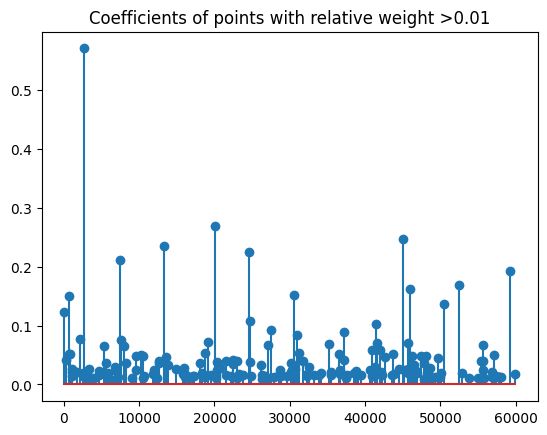} 
\end{tabular}
\caption{Rates and thresholding of optimal measure $Q_{n}$, for \cref{fig:SP_noise_8} }
\label{fig:sp_noise_8_rates}
\end{figure}

Examining \cref{fig:Gaus_noise_8,fig:SP_noise_8}, we can make several observations which will persist globally in all numerics. As seen in \cref{fig:Gaus_noise_8_rates,fig:sp_noise_8_rates}, the theoretical rate $n^{1/4}$ agrees with practical results, and appears to be tight for moderate $n$ on the order of $n < 30000$. While the leading constants given by theory are quite large, growing polynomially with $d, \rho$ e.g., in practice these are much more moderate - here $O(1)$. We also note the linear combination forming the final solution is quite compressible, here having only 157 datapoints having dual measure above $0.01$, and after thresholding being visually indistinguishable from the full solution. Furthermore the \textit{visual} convergence is fast, in the sense that there is not an appreciable visual difference between the solution given 20,000 and 60,000 datapoints. As they are  formed as linear combinations of observed data, all  solutions suffer from artefacting in the form of blurred edges, which can be solved by a final mask at the pixel level. 

Before shifting away from the MNIST dataset, we also include a cautionary experiment where, for small random samples, the method is visually ``confidently incorrect'', which can be seen in \Cref{fig:vis_switches}. In the previous examples \cref{fig:SP_noise_8,fig:Gaus_noise_8} for small $n$ we simply have blurry images. This type of failure is generally representative of how the method performs when $n$ is too small. However there is another failure case which distinctly highlights the risk of taking $n$ too small, which can lead to a biased sample which does not well approximate $\mu$ - leading to correspondingly biased solutions. In \cref{fig:vis_switches} we see the recovered image after masking is clearly a $3$ for one random samples of size $n=1000$, but a $5$ for other random samples of size $n=800,2000$. 
\begin{figure}[h]
\begin{tabular}{ccccc}
  \includegraphics[width=0.15\textwidth]{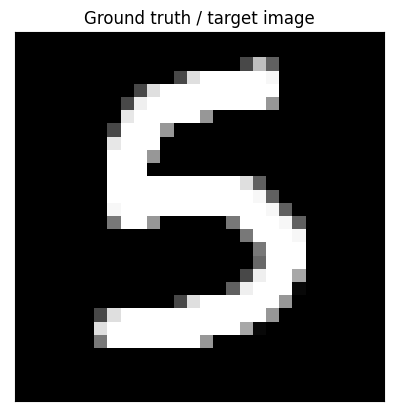} &   \includegraphics[width=0.15\textwidth]{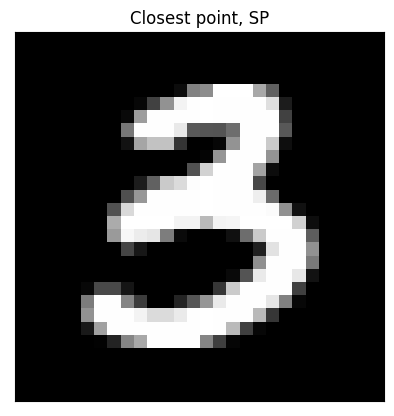} &
  \includegraphics[width=0.15\textwidth]{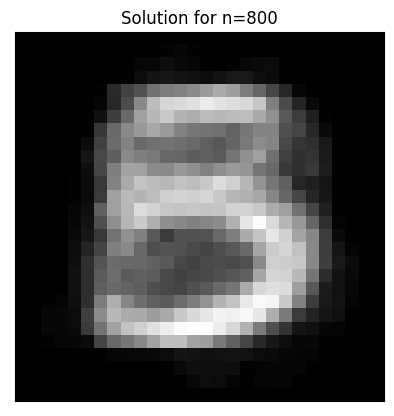} & \includegraphics[width=0.15\textwidth]{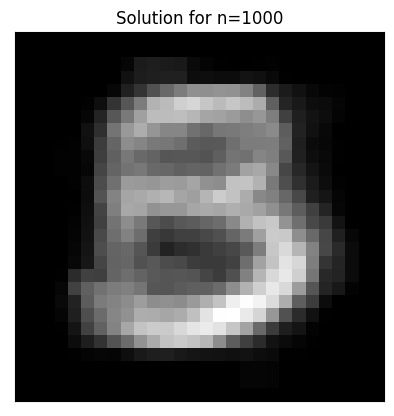} &
  \includegraphics[width=0.15\textwidth]{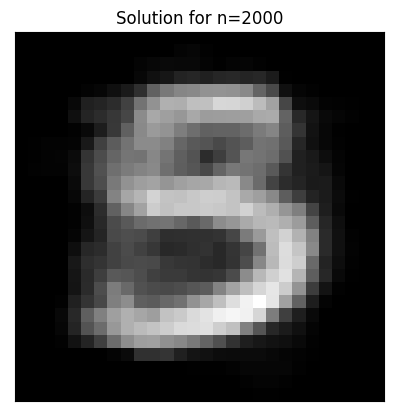} \\
  & \includegraphics[width=0.15\textwidth]{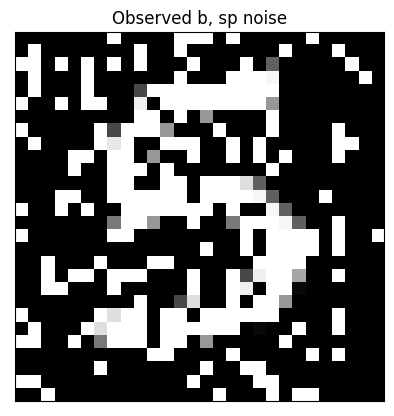}  & 
  \includegraphics[width=0.15\textwidth]{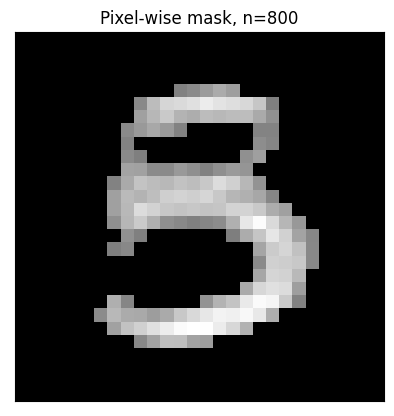} & \includegraphics[width=0.15\textwidth]{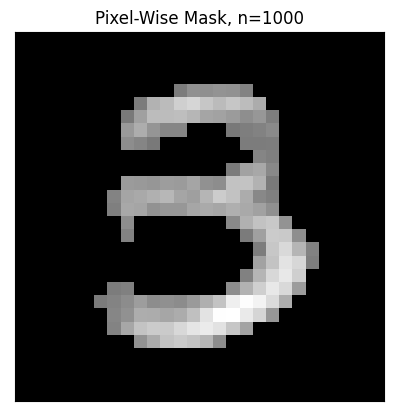} &
  \includegraphics[width=0.15\textwidth]{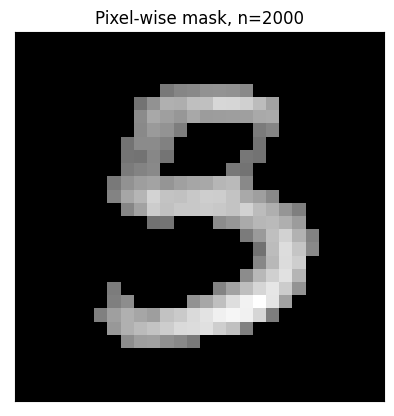}
\end{tabular}
\caption{A specific type of failure case for small $n$.}
\label{fig:vis_switches}
\end{figure}

We now move to the more expressive MNIST fashion dataset. We once again use a hand drawn target, which is not originally found in the dataset. To emphasize differences compared to handwritten digits, the fashion dataset is immediately more challenging, especially in regards to fine details. To illustrate, there are few examples of garments with spots, stripes, or other detailed patterns - and as such are much more difficult to learn from uniform random samples. Similarly, if the target ground truth image contains details or patterns which are not present in the fashion dataset, there is little hope of constructing a reasonable linear combination to approximate the ground truth. On the other hand, some classes - such as heels or sandals - are extremely easy to learn, as there are many near-identical examples in the dataset, and are visually quite distinct from many other classes. Disregarding the change in dataset, our methodology for remains the same as \cref{fig:Gaus_noise_8}

\begin{figure}[h]
\centering
\begin{tabular}{cccc}
    \includegraphics[width = 0.17\textwidth]{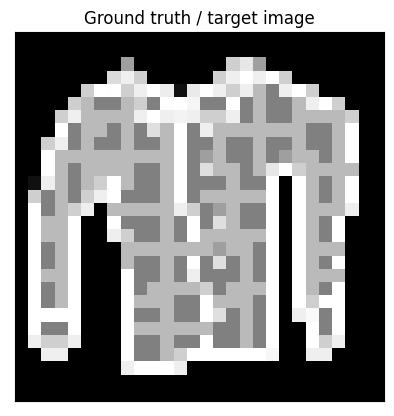} &
    \includegraphics[width = 0.17\textwidth]{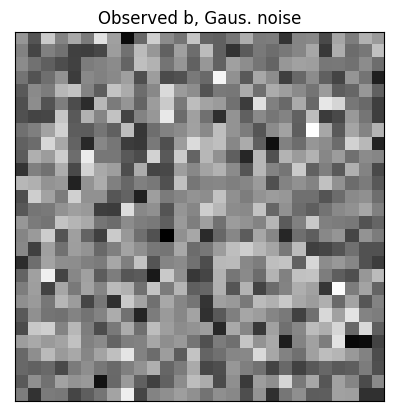} & 
    \includegraphics[width = 0.17\textwidth]{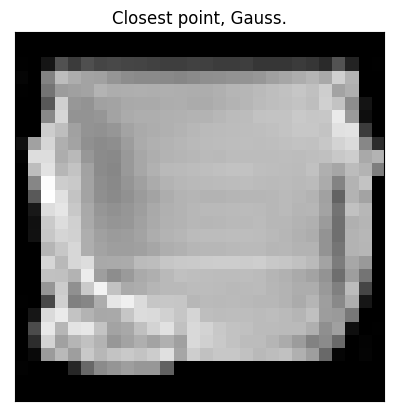} &
    \includegraphics[width = 0.17\textwidth]{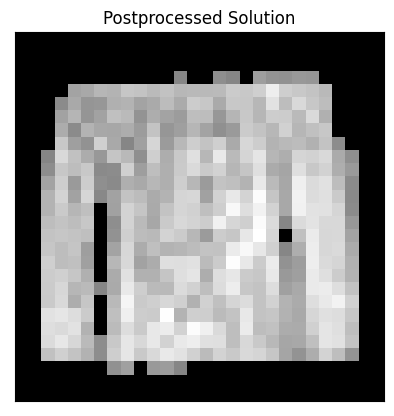}  \\ 
    \includegraphics[width = 0.2\textwidth]{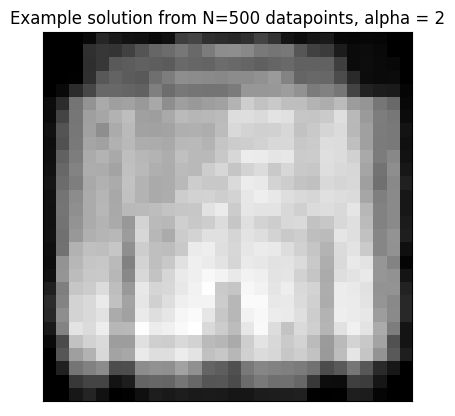} &
     \includegraphics[width = 0.2\textwidth]{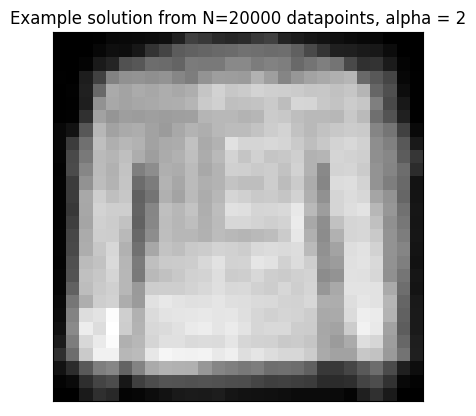} &
     \includegraphics[width = 0.2\textwidth]{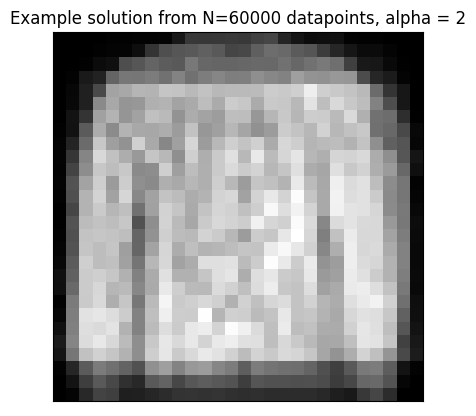} &
     \includegraphics[width = 0.17\textwidth]{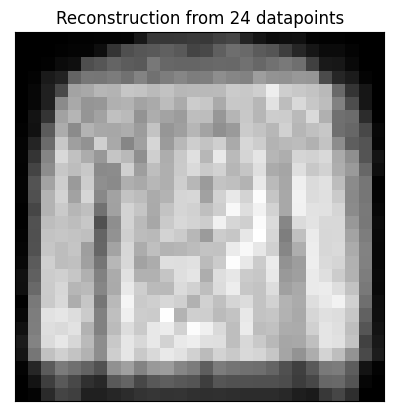} 
\end{tabular}
    \caption{Recovery of a hand drawn shirt with additive Gaussian noise}
    \label{fig:Gaus_noise_shirt}
\end{figure}

\begin{figure}[h]
\centering
\begin{tabular}{cc} 
    \includegraphics[width = 0.4\textwidth]{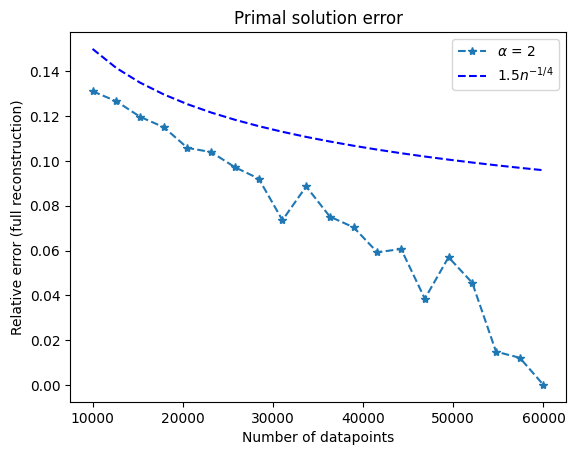} &
    \includegraphics[width = 0.4\textwidth]{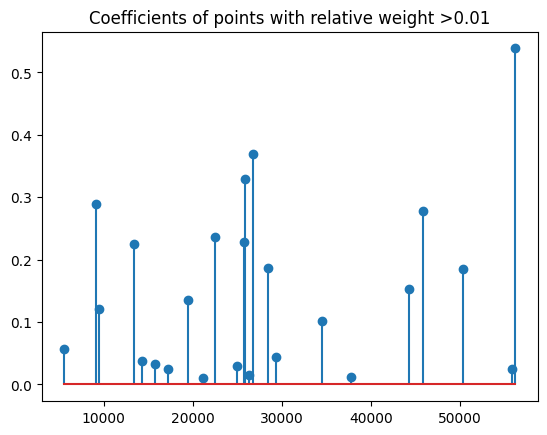} 
\end{tabular}
\caption{Rates and thresholding of optimal measure $Q_{n}$, for \cref{fig:Gaus_noise_shirt}}
\label{fig:Gaus_noise_shirt_rates}
\end{figure}

\begin{figure}[h]
\centering
\begin{tabular}{cccc}
    \includegraphics[width = 0.17\textwidth]{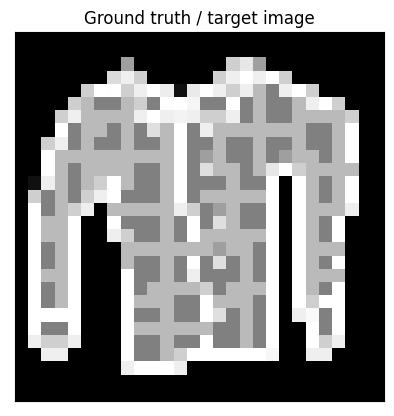} &
    \includegraphics[width = 0.17\textwidth]{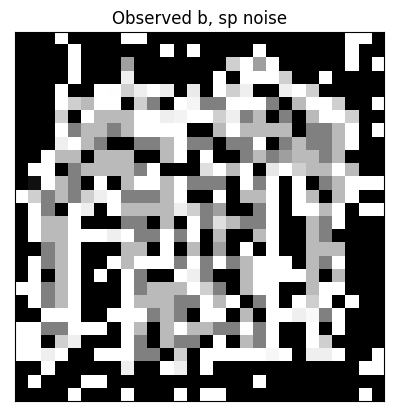} & 
    \includegraphics[width = 0.17\textwidth]{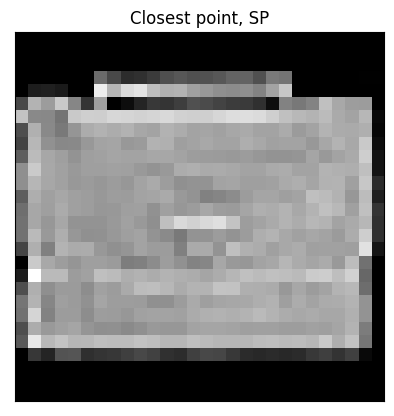} &
    \includegraphics[width = 0.17\textwidth]{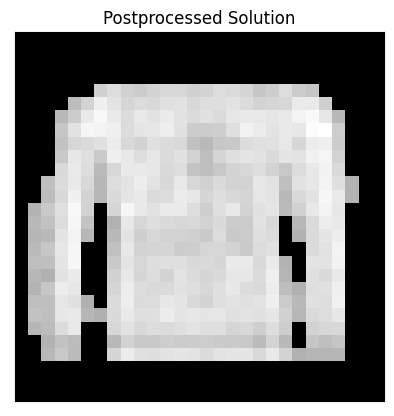}\\ 
    \includegraphics[width = 0.2\textwidth]{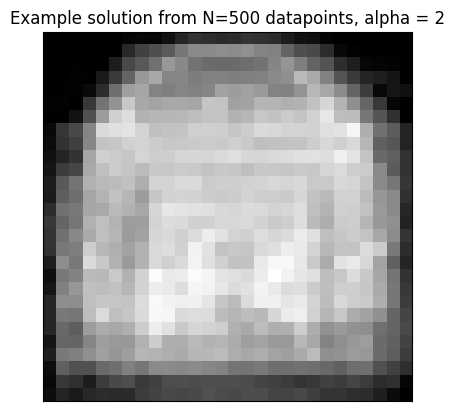} &
     \includegraphics[width = 0.2\textwidth]{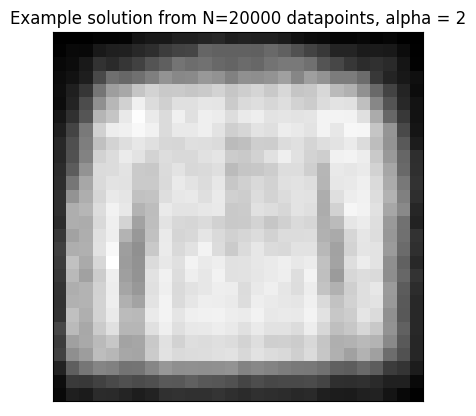} &
     \includegraphics[width = 0.2\textwidth]{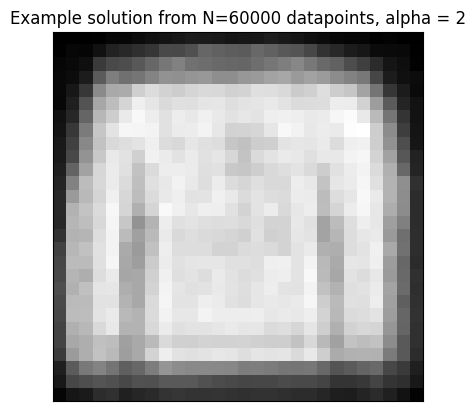} &
     \includegraphics[width = 0.17\textwidth]{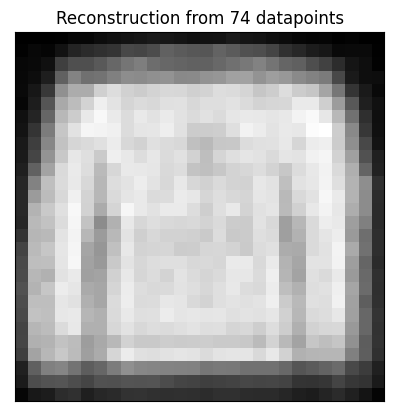} 
\end{tabular}
    \caption{Recovery of a hand drawn shirt with salt and pepper corruption noise.}
    \label{fig:SP_noise_shirt}
\end{figure}

\begin{figure}[h]
\centering
\begin{tabular}{cc} 
    \includegraphics[width = 0.4\textwidth]{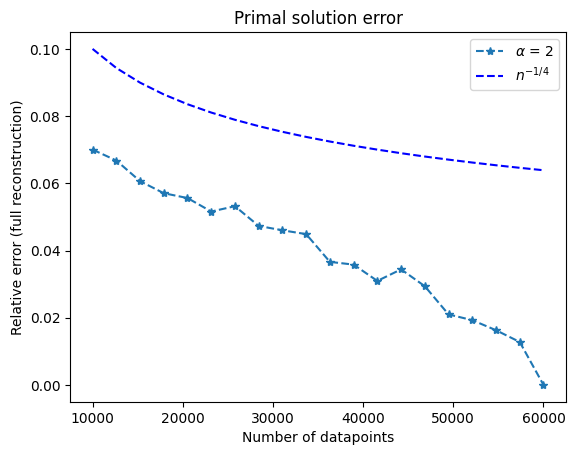} 
    \includegraphics[width = 0.4\textwidth]{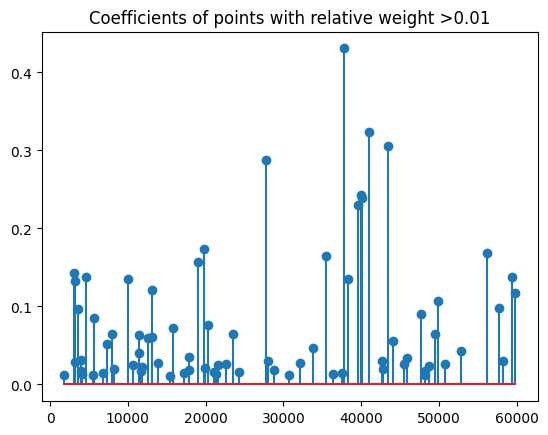} 
\end{tabular}
\caption{Rates and thresholding of optimal measure $Q_{n}$, for \cref{fig:SP_noise_shirt}}
\label{fig:SP_noise_shirt_rates}
\end{figure}

For the experiment with salt-and-pepper noise \cref{fig:SP_noise_shirt}, while MEM denoising clearly recovers a shirt, many of the finer details are lost or washed out. In contrast, with Gaussian noise \cref{fig:Gaus_noise_shirt}, there is some remnant of the shirts' pattern visible in the final reconstruction. In both cases the nearest neighbor is a bag or purse, and not visually close to the ground truth. While the constant leading constant is larger, once again we are firmly below the theoretically expected convergence rate of $n^{1/4}$, and solutions are compressible with respect to the optimal measure $Q_{n}$.

Finally, we conclude with an experiment on MNIST fashion with a hand drawn target of a heel, where we observe once again recovery well within the expected convergence rate, and visually recovers (after postprocessing) a reasonable approximation to the ground truth.

\begin{figure}[h]
\centering
\begin{tabular}{cccc}
    \includegraphics[width = 0.17\textwidth]{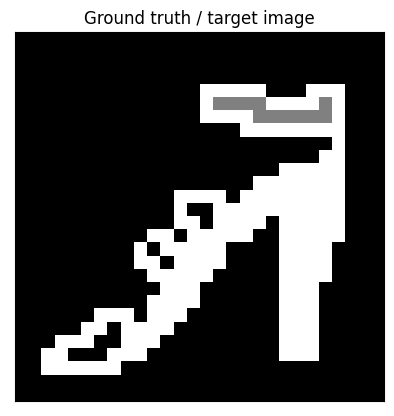} &
    \includegraphics[width = 0.17\textwidth]{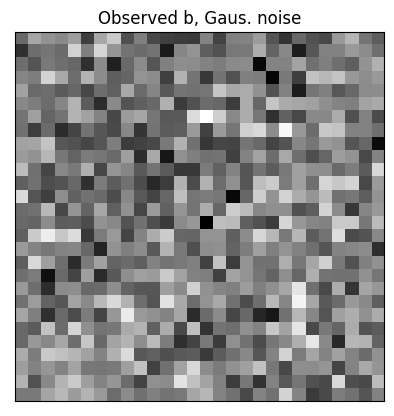} & 
    \includegraphics[width = 0.17\textwidth]{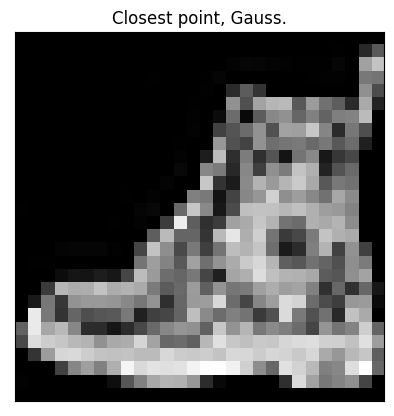} &
    \includegraphics[width = 0.17\textwidth]{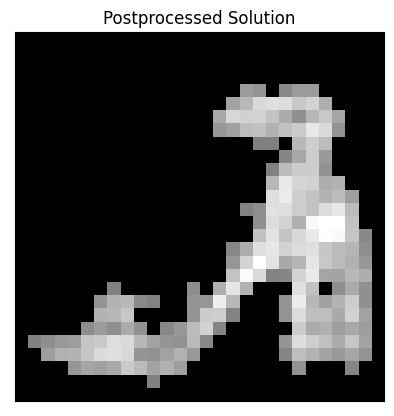} \\ 
    \includegraphics[width = 0.2\textwidth]{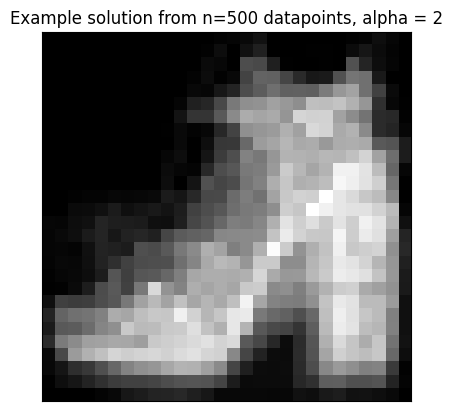} &
     \includegraphics[width = 0.2\textwidth]{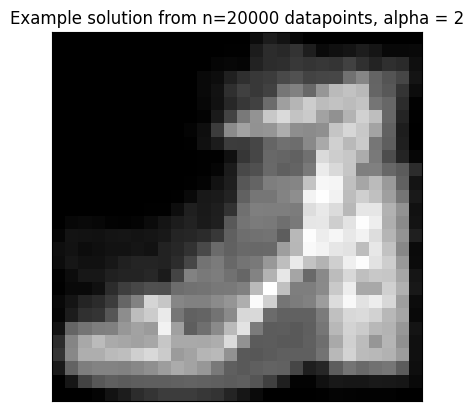} &
     \includegraphics[width = 0.2\textwidth]{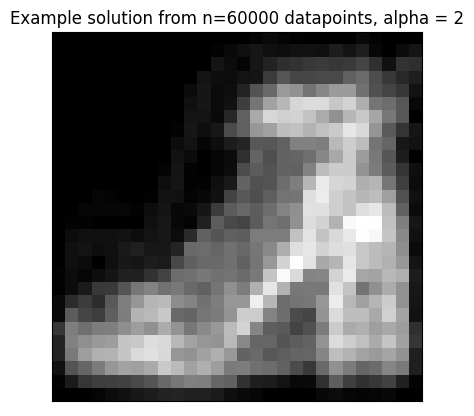} &
     \includegraphics[width = 0.17\textwidth]{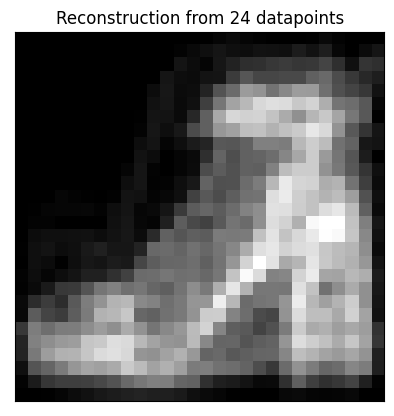} 
\end{tabular}
    \caption{Recovery hand drawn heel with additive Gaussian noise.}
    \label{fig:gaus_noise_heel}
\end{figure}

\begin{figure}[h]
\centering
\begin{tabular}{cc} 
    \includegraphics[width = 0.4\textwidth]{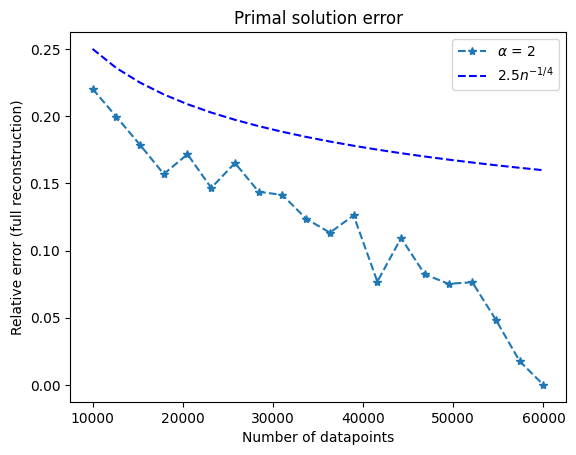}&
    \includegraphics[width = 0.4\textwidth]{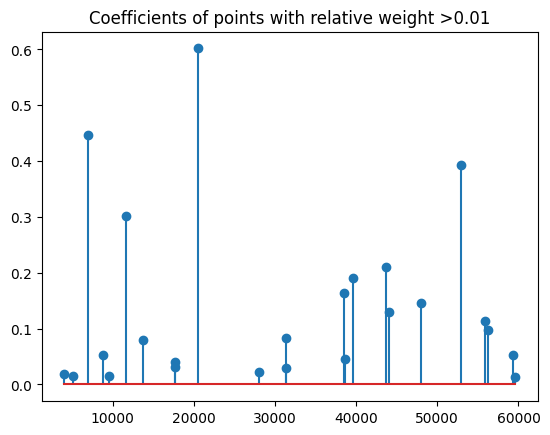} 
\end{tabular}
\caption{Rates and thresholding of optimal measure $Q_{n}$, for \cref{fig:gaus_noise_heel}}
\label{fig:gaus_noise_heel_rates}
\end{figure}

\begin{figure}[h]
\centering
\begin{tabular}{cccc}
    \includegraphics[width = 0.17\textwidth]{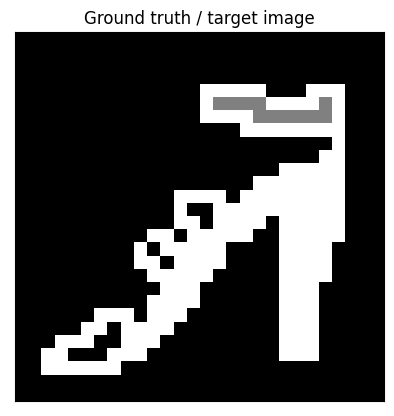} &
    \includegraphics[width = 0.17\textwidth]{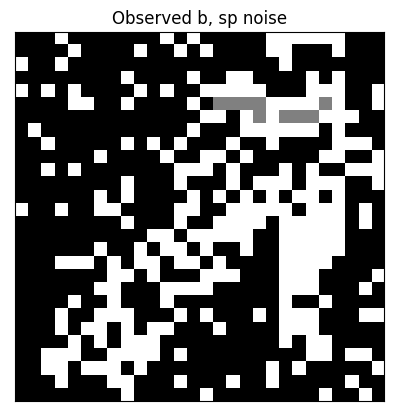} & 
    \includegraphics[width = 0.17\textwidth]{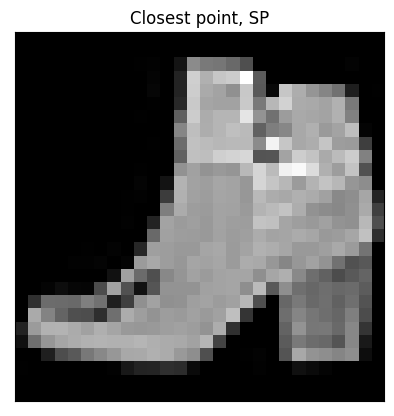} &
    \includegraphics[width = 0.17\textwidth]{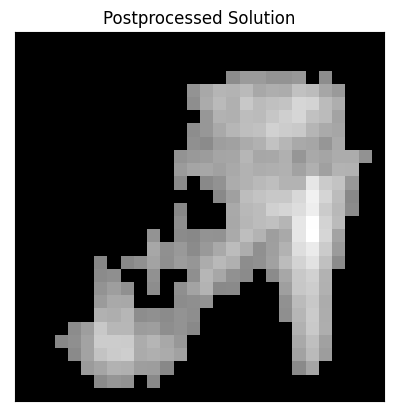}  \\ 
    \includegraphics[width = 0.2\textwidth]{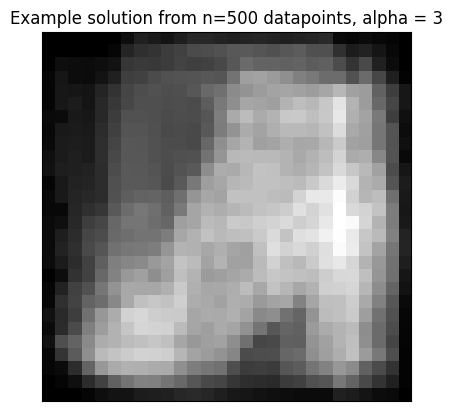} &
     \includegraphics[width = 0.2\textwidth]{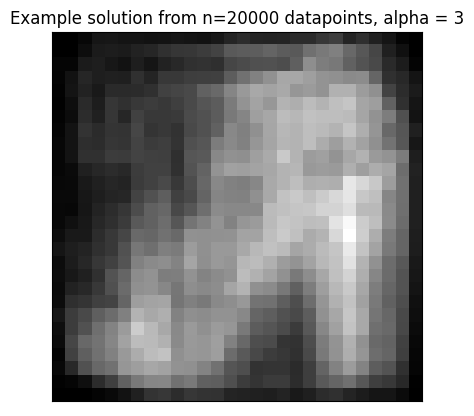} &
     \includegraphics[width = 0.2\textwidth]{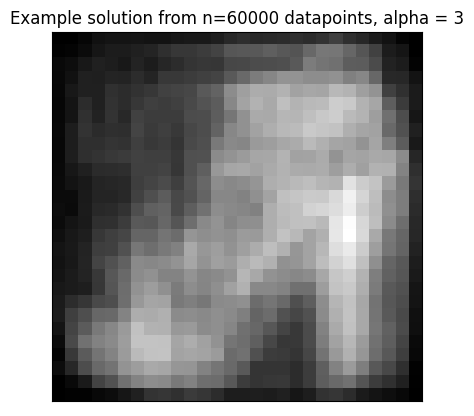} &
     \includegraphics[width = 0.17\textwidth]{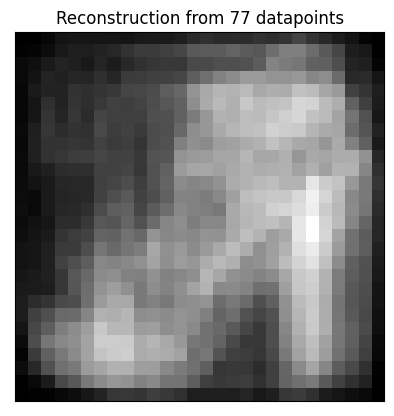} 
\end{tabular}
    \caption{Recovery of hand drawn heel with Salt and Pepper corruptions.}
    \label{fig:SP_noise_heel}
\end{figure}

\begin{figure}[h]
\centering
\begin{tabular}{cc} 
    \includegraphics[width = 0.4\textwidth]{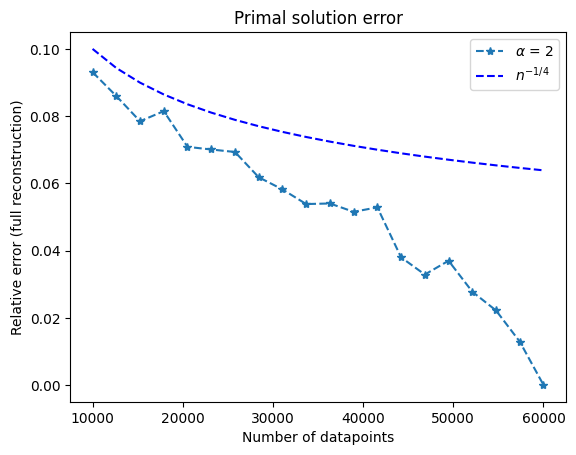}&
    \includegraphics[width = 0.4\textwidth]{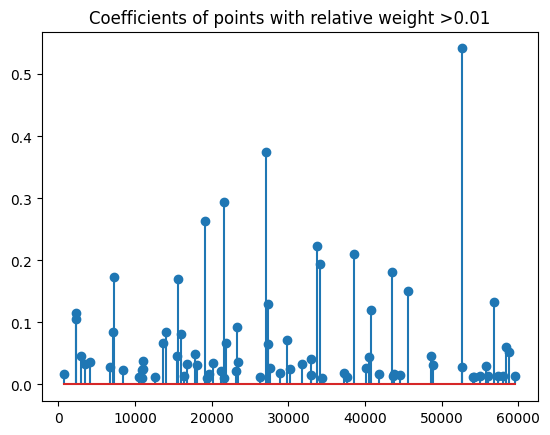} 
\end{tabular}
\caption{Rates and thresholding of optimal measure $Q_{n}$, for \cref{fig:SP_noise_heel}}
\label{fig:SP_noise_heel_rates}
\end{figure}

\clearpage

\bibliographystyle{siam}
\bibliography{references}

\begin{thebibliography}{10}

\bibitem{attouch1977convergence}
{\sc H.~Attouch}, {\em Convergence de fonctions convexes, des sous-diff{\'e}rentiels et semi-groupes associ{\'e}s}, CR Acad. Sci. Paris, 284 (1977), p.~13.

\bibitem{attouch1984variational}
\leavevmode\vrule height 2pt depth -1.6pt width 23pt, {\em Variational convergence for functions and operators}, Applicable Mathematics Series,  (1984).

\bibitem{auslender2006interior}
{\sc A.~Auslender and M.~Teboulle}, {\em Interior gradient and proximal methods for convex and conic optimization}, SIAM Journal on Optimization, 16 (2006), pp.~697--725.

\bibitem{bauschke2019convex}
{\sc H.~Bauschke and P.~Combettes}, {\em Convex analysis and monotone operator theory in hilbert spaces, corrected printing}, 2019.

\bibitem{billingsley2017probability}
{\sc P.~Billingsley}, {\em Probability and measure}, John Wiley \& Sons, 2017.

\bibitem{BuH15}
{\sc J.~Burke and T.~Hoheisel}, {\em A note on epi-convergence of sums under the inf-addition rule}, Optimization Online,  (2015).

\bibitem{byrd1994representations}
{\sc R.~H. Byrd, J.~Nocedal, and R.~B. Schnabel}, {\em Representations of quasi-newton matrices and their use in limited memory methods}, Mathematical Programming, 63 (1994), pp.~129--156.

\bibitem{csorgo1982empirical}
{\sc S.~Cs{\"o}rg{\H{o}}}, {\em The empirical moment generating function}, BV Gnedenko, ML Puri \& I. Vincze, 32 (1982), pp.~139--150.

\bibitem{csorgHo1983kernel}
\leavevmode\vrule height 2pt depth -1.6pt width 23pt, {\em Kernel-transformed empirical processes}, Journal of multivariate analysis, 13 (1983), pp.~517--533.

\bibitem{deng2012mnist}
{\sc L.~Deng}, {\em The mnist database of handwritten digit images for machine learning research [best of the web]}, IEEE signal processing magazine, 29 (2012), pp.~141--142.

\bibitem{donsker1976asymptotic3}
{\sc M.~D. Donsker and S.~S. Varadhan}, {\em Asymptotic evaluation of certain {Markov} process expectations for large time—iii}, Communications on pure and applied Mathematics, 29 (1976), pp.~389--461.

\bibitem{feuerverger1989empirical}
{\sc A.~Feuerverger}, {\em On the empirical saddlepoint approximation}, Biometrika, 76 (1989), pp.~457--464.

\bibitem{folland1999real}
{\sc G.~B. Folland}, {\em Real analysis: modern techniques and their applications}, vol.~40, John Wiley \& Sons, 1999.

\bibitem{jaynes1957information1}
{\sc E.~T. Jaynes}, {\em {Information Theory and Statistical Mechanics}}, Phys. Rev., 106 (1957), pp.~620--630.

\bibitem{jaynes1957information2}
\leavevmode\vrule height 2pt depth -1.6pt width 23pt, {\em {Information Theory and Statistical Mechanics. II}}, Phys. Rev., 108 (1957), pp.~171--190.

\bibitem{Kantas2015Particle}
{\sc N.~Kantas et~al.}, {\em {On Particle Methods for Parameter Estimation in State-Space Models}}, Statistical Science, 30 (2015), pp.~328 -- 351.

\bibitem{king1991epi}
{\sc A.~J. King and R.~Wets}, {\em Epi-consistency of convex stochastic programs}, Stochastics and Stochastic Reports, 34 (1991), pp.~83--92.

\bibitem{klenke2013probability}
{\sc A.~Klenke}, {\em Probability Theory: a Comprehensive Course}, Springer Science \& Business Media, 2013.

\bibitem{kullback1951information}
{\sc S.~Kullback and R.~Leibler}, {\em On information and sufficiency}, The annals of mathematical statistics, 22 (1951), pp.~79--86.

\bibitem{le1999new}
{\sc G.~Le~Besnerais, J.-F. Bercher, and G.~Demoment}, {\em A new look at entropy for solving linear inverse problems}, IEEE Transactions on Information Theory, 45 (1999), pp.~1565--1578.

\bibitem{rioux2020maximum}
{\sc G.~Rioux, R.~Choksi, T.~Hoheisel, P.~Marechal, and C.~Scarvelis}, {\em The maximum entropy on the mean method for image deblurring}, Inverse Problems, 37 (2021), p.~015011.

\bibitem{8758192}
{\sc G.~Rioux, C.~Scarvelis, R.~Choksi, T.~Hoheisel, and P.~Maréchal}, {\em Blind deblurring of barcodes via kullback-leibler divergence}, IEEE Transactions on Pattern Analysis and Machine Intelligence, 43 (2021), pp.~77--88.

\bibitem{rockafellar1997convex}
{\sc R.~T. Rockafellar}, {\em Convex Analysis}, vol.~18, Princeton University Press, 1997.

\bibitem{rockafellar2006variational}
{\sc R.~T. Rockafellar and R.~Wets}, {\em Variational systems, an introduction}, in Multifunctions and Integrands: Stochastic Analysis, Approximation and Optimization Proceedings of a Conference held in Catania, Italy, June 7--16, 1983, Springer, 2006, pp.~1--54.

\bibitem{rockafellar2009variational}
\leavevmode\vrule height 2pt depth -1.6pt width 23pt, {\em Variational analysis}, vol.~317, Springer Science \& Business Media, 2009.

\bibitem{romisch2007stability}
{\sc W.~R{\"o}misch and R.~Wets}, {\em Stability of $\varepsilon$-approximate solutions to convex stochastic programs}, SIAM Journal on Optimization, 18 (2007), pp.~961--979.

\bibitem{royset2022optimization}
{\sc J.~Royset and R.~Wets}, {\em An Optimization Primer}, Springer, 2022.

\bibitem{severini2005elements}
{\sc T.~Severini}, {\em Elements of distribution theory}, vol.~17, Cambridge University Press, 2005.

\bibitem{vaisbourd2022maximum}
{\sc Y.~Vaisbourd, R.~Choksi, A.~Goodwin, T.~Hoheisel, and C.-B. Sch{\"o}nlieb}, {\em Maximum entropy on the mean and the {Cram\'er} rate function in statistical estimation and inverse problems: Properties, models, and algorithms}, arXiv preprint arXiv:2211.05205,  (2023).
\newblock To appear in Mathematical Programming.

\bibitem{van1996new}
{\sc A.~Van Der~Vaart}, {\em New {Donsker} classes}, The Annals of Probability, 24 (1996), pp.~2128--2140.

\bibitem{van1996weak}
{\sc A.~Van Der~Vaart and J.~Wellner}, {\em Weak convergence}, Springer, 1996.

\bibitem{xiao2017fashion}
{\sc H.~Xiao}, {\em Fashion-mnist: a novel image dataset for benchmarking machine learning algorithms}, arXiv preprint arXiv:1708.07747,  (2017).

\end{thebibliography}

\end{document}